%% file: main.tex
\documentclass[twoside]{article}

\usepackage{aistats2025}

\usepackage[square]{natbib}

\input{header}
\input{symbols}

\begin{document}

%

%

\twocolumn[

\aistatstitle{Fast Convergence of Softmax Policy Mirror Ascent}

\aistatsauthor{ Reza Asad\textsuperscript{1} \And Reza Babanezhad\textsuperscript{2} \And Issam Laradji\textsuperscript{3} \And Nicolas Le Roux\textsuperscript{4} \And Sharan Vaswani\textsuperscript{1}}

\aistatsaddress{\textsuperscript{1}Simon Fraser University \quad \textsuperscript{2} Samsung AI \quad \textsuperscript{3} ServiceNow \quad \textsuperscript{4}Mila, Universit\'e de Montr\'eal, McGill } ]

\input{abstract}
\input{introduction}

\input{problem-formulation}

\input{algorithm}

\input{theory}

\input{fa}

\input{experiments}

\input{conclusion}

\input{acknowledgement}

\clearpage
\bibliography{ref}

\clearpage
\appendix
\thispagestyle{empty}
\onecolumn
\input{appendix-setup}
\input{appendix-proofs}

\clearpage
\input{appendix-mdp-results}

\input{appendix-stable-baseline-details}

\end{document}

%% file: header.tex
\usepackage[utf8]{inputenc} 
\usepackage[T1]{fontenc}    
\usepackage{url}            
\usepackage{booktabs}       
\usepackage{nicefrac}       
\usepackage{microtype}      
\usepackage{xcolor}         
\usepackage{wrapfig}
\usepackage{subcaption}
\usepackage[ruled]{algorithm2e}
\usepackage{algorithmic}
\usepackage{amssymb,amsmath}
\usepackage{amsthm}
\usepackage{thmtools, thm-restate}
\usepackage[mathscr]{eucal} 
\usepackage{natbib}
\usepackage{bbm}
\usepackage{arydshln}
\usepackage{enumitem}
\usepackage{cancel}
\usepackage{graphicx}
\usepackage{mdframed}

\definecolor{mydarkblue}{rgb}{0,0.08,0.45}
\usepackage[colorlinks=true,
    linkcolor=mydarkblue,
    citecolor=mydarkblue,
    filecolor=mydarkblue,
    urlcolor=mydarkblue]{hyperref} 
\usepackage[capitalize]{cleveref}

\declaretheorem{lemma}

\declaretheorem{assumption}
\crefname{assumption}{assumption}{Assumptions}
\crefname{figure}{Figure}{Figure}

\mdfdefinestyle{mdframedthmbox}{%
	leftmargin=.0\textwidth,
	rightmargin=.0\textwidth,%
	innertopmargin=0.75em,
	innerleftmargin=.5em,
	innerrightmargin=.5em,
}

\SetAlgoSkip{}

\usepackage{pifont}  
\newcommand{\cmark}{\textcolor{green}{\ding{51}}}
\newcommand{\xmark}{\textcolor{red}{\ding{55}}} 

%% file: symbols.tex
\newcommand{\Alg}{\texttt{SPMA}}
\newcommand{\NPG}{\texttt{NPG}}
\newcommand{\MDPO}{\texttt{MDPO}}
\newcommand{\PPO}{\texttt{PPO}}
\newcommand{\SPG}{\texttt{SPG}}

\newcommand{\TRPO}{\texttt{TRPO}}
\newcommand{\TRPOC}{\texttt{TRPO-constrained}}
\newcommand{\TRPOR}{\texttt{TRPO-regularized}}

\newcommand{\X}{\textbf{X}}

\newcommand{\z}{z}
\newcommand{\zt}{z_{t}}

\newcommand{\ztt}{z_{t+1}}
\newcommand{\zth}{z_{t+\nicefrac{1}{2}}}

\newcommand{\inner}[2]{\langle #1, #2 \rangle}

\newcommand{\norm}[1]{\left\|#1\right\|}

\newcommand{\E}{\mathbb{E}}

\DeclareMathOperator*{\argmax}{arg\,max}
\DeclareMathOperator*{\argmin}{arg\,min}

\def \nus/{\texttt{NUS}}

\newcommand{\nleft}{\mathclose\bgroup\left}
\newcommand{\nright}{\aftergroup\egroup\right}

\def\1{\bm{1}}

\DeclareMathAlphabet{\mathsfit}{\encodingdefault}{\sfdefault}{m}{sl}
\SetMathAlphabet{\mathsfit}{bold}{\encodingdefault}{\sfdefault}{bx}{n}

\newcommand{\R}{\mathbb{R}}

\newcommand{\tightsub}[1]{{\kern -.1em \raise-.1em\hbox{\tiny$#1$}}{}}

\DeclareMathAlphabet{\mathsfit}{\encodingdefault}{\sfdefault}{m}{sl}
\SetMathAlphabet{\mathsfit}{bold}{\encodingdefault}{\sfdefault}{bx}{n}

\def\cP{{\mathcal{P}}}

\def\cS{{\mathcal{S}}}

\def\ppi{{\pi}}
\def\ppit{{\pi_t}}
\def\ppith{{\pi_{t + \nicefrac{1}{2}}}}
\def\ppitt{\pi_{t+1}}

\def\dpi{{d^{\pi}}}

\def\dpit{{d^{\pi_t}}}

\def\dpitt{{d^{\pi_{t+1}}}}

\def\qpi{{Q^{\pi}}}
\def\qpit{{Q^{\pi_t}}}

\def\tpi{{T^\pi}}

\def\tpitt{{T^{\pi_{t+1}}}}

\def\jspit{{V^{\pi_t}(s)}}

\def\dbpitt{d^{\pi_{t+1/2}}}

\def\qpitt{{Q^{\pi_{t+1}}}}
\def\bppitt{{\pi_{t+1/2}}}

\def\api{{A^{\pi}}}
\def\apit{{A^{\pi_{t}}}}
\def\hapit{{\hat{A}^{\pi_{t}}}}

\def\prpi{{\Pr^{\pi}}}

\DeclareMathOperator{\sign}{sign}

\def\cA{{\mathcal{A}}}

\def\cM{{\mathcal{M}}}

\def\cH{{\mathcal{H}}}

\newcommand{\tht}{\theta_{t}}
\newcommand{\thtt}{\theta_{t+1}}

\newcommand{\norminf}[1]{\left\| #1\right\|_{\infty}}

\newcommand{\cZ}{\mathcal{Z}}

%% file: abstract.tex
\begin{abstract}
Natural policy gradient ($\NPG$) is a common policy optimization algorithm and can be viewed as mirror ascent in the space of probabilities. Recently,~\citet{vaswani2021general} introduced a policy gradient method that corresponds to mirror ascent in the dual space of logits. We refine this algorithm, removing its need for a normalization across actions and analyze the resulting method (referred to as $\Alg$). For tabular MDPs, we prove that $\Alg$ with a constant step-size matches the linear convergence of $\NPG$ and achieves a faster convergence than constant step-size (accelerated) softmax policy gradient. To handle large state-action spaces, we extend $\Alg$ to use a log-linear policy parameterization. Unlike that for $\NPG$, generalizing $\Alg$ to the linear function approximation (FA) setting does not require compatible function approximation. Unlike $\MDPO$, a practical generalization of $\NPG$, $\Alg$ with linear FA only requires solving convex softmax classification problems. We prove that $\Alg$ achieves linear convergence to the neighbourhood of the optimal value function. We extend $\Alg$ to handle non-linear FA and evaluate its empirical performance on the MuJoCo and Atari benchmarks. Our results demonstrate that $\Alg$ consistently achieves similar or better performance compared to $\MDPO$, $\PPO$ and $\TRPO$.
\end{abstract}

%% file: introduction.tex
\section{INTRODUCTION}
\label{sec:introduction}
Policy gradient (PG) methods~\citep{williams1992simple,sutton1999policy,konda2000actor,kakade2001natural} have been critical to the achievements of reinforcement learning (RL). Although the PG objective is non-concave, recent theoretical research~\citep{agarwal2021theory,mei2020global,mei2021leveraging,bhandari2021linear,lan2023policy,shani2020adaptive,liu2024elementary,lu2024towards,alfano2022linear,yuan2023linear} has analyzed PG methods in simplified settings and demonstrated their global convergence to an optimal policy. While such simplified analyses are helpful in understanding the underlying optimization issues, the resulting methods are rarely used in practice. On the other hand, while methods such as \texttt{TRPO}~\citep{schulman2015trust}, \texttt{PPO}~\citep{schulman2017proximal}, \texttt{MDPO}~\citep{tomar2020mirror} are commonly used in deep RL, their theoretical analysis in the function approximation setting is quite limited. In particular, existing work either (i) analyzes these methods only in the impractical tabular setting~\citep{tomar2020mirror,shani2020adaptive} or (ii) modifies these algorithms to make them more amenable to theoretical analysis~\citep{liu1906neural,zhong2024theoretical}. Unfortunately, these modified algorithms are quite different from the original variants and are not systematically benchmarked on standard environments. Consequently, there exists a large gap between PG methods that have theoretical guarantees in realistic settings versus those which are implemented in practice. To make matters worse, it has been demonstrated that code-level implementation details impact the empirical performance more than the underlying algorithm~\citep{engstrom2019implementation}. 

Designing theoretically principled PG algorithms that simultaneously have good empirical performance on the standard set of benchmarks is the main motivation behind this work. To that end, we leverage an algorithm first proposed by~\citet{vaswani2021general}, which we modify to remove the need for normalization. We coin this refinement \textbf{S}oftmax \textbf{P}olicy \textbf{M}irror \textbf{A}scent (referred to as $\Alg$). We show that $\Alg$ has comparable convergence guarantees as existing theoretical techniques~\citep{lu2024towards,yuan2023linear} in the tabular and function approximation settings, while achieving comparable practical performance as \texttt{PPO}, \texttt{TRPO} and \texttt{MDPO}, without additional algorithmic modifications. In particular, we make the following contributions. 

\textbf{Contribution 1}: In~\cref{sec:algorithm}, we focus on the multi-armed bandit and tabular MDP settings, where the number of parameters scales with the number of states and actions. We develop the $\Alg$ algorithm, which parameterizes the policy using the softmax function and uses a mirror ascent (with the log-sum-exp mirror map) update. Compared to $\NPG$ that can be viewed as mirror ascent in the space of probabilities, $\Alg$ corresponds to mirror ascent in the dual space of logits and does not require a normalization across actions. Given access to the exact policy gradients, we prove that $\Alg$ with a constant step-size converges to the optimal policy at a linear rate and thus matches the rate of $\NPG$~\citep{khodadadian2021linear,liu2024elementary}. In comparison, constant step-size softmax policy gradient ($\SPG$)~\citep{agarwal2021theory,mei2020global} can only achieve sublinear convergence rates even with Nesterov acceleration~\citep{chenaccelerated}. Hence, by changing the mirror map (from Euclidean to log-sum-exp) while using the same policy parameterization, $\Alg$ can result in an exponential improvement over $\SPG$. 

\textbf{Contribution 2}: In order to handle MDPs with large state-action spaces, we use function approximation (e.g. linear models or neural networks) to parameterize the policies resulting in the class of log-linear or energy-based policies~\citep{haarnoja2017reinforcement,agarwal2021theory,yuan2023linear} respectively. By interpreting the policy parameterization as a constraint on the corresponding logits, we use projected mirror ascent to extend $\Alg$ to the FA setting and design~\cref{alg:spma}. Unlike that for $\NPG$, generalizing $\Alg$ does not require compatible function approximation, and thus results in a more practical algorithm. Unlike $\MDPO$~\citep{tomar2020mirror} which results in non-convex surrogates for linear FA, $\Alg$ requires solving convex softmax classification problems in each iteration. 

\textbf{Contribution 3}: In~\cref{sec:fa-theory}, we state the conditions under which~\cref{alg:spma}  converges to the neighbourhood of the optimal value function, and characterize the resulting linear convergence rate. Hence, for log-linear policies,~\cref{alg:spma} matches the theoretical convergence of $\NPG$ with compatible function approximation~\citep{agarwal2021theory,alfano2022linear,yuan2023linear}. Our theoretical results are better than those in~\citet{vaswani2021general} and~\citet{schulman2015trust} which prove sublinear convergence to a stationary point for idealized variants of $\Alg$ and \texttt{TRPO} respectively. In contrast to~\citet{kuba2022mirror} which prove that the idealized variants of $\PPO$ and $\TRPO$ converge to the optimal policy asymptotically, we characterize the non-asymptotic convergence rate for~\cref{alg:spma}. 

\textbf{Contribution 4}: We empirically evaluate $\Alg$ across simple MDPs with tabular and linear parameterization, Atari games with a discrete action space and a neural policy parameterization with CNNs, and continuous control MuJoCo tasks with a continuous action space and a neural policy parameterization with MLPs. We demonstrate that $\Alg$ has consistently good performance -- on Atari games $\Alg$ achieves better results than both $\TRPO$ and $\PPO$ while matching or outperforming $\MDPO$, whereas on MuJoCo tasks, $\Alg$ outperforms $\PPO$ and achieves similar or better results than $\MDPO$. 





%% file: problem-formulation.tex
\vspace{-2ex}
\section{PROBLEM FORMULATION}
\vspace{-2ex}
We consider an infinite-horizon discounted Markov decision process (MDP)~\citep{puterman2014markov} defined by $\cM = \langle \cS , \cA, \cP, r, \rho, \gamma \rangle$, where $\cS$ and $\cA$ represent the states and actions, $\mathcal{P}: \cS \times \cA \rightarrow \Delta_{\cS}$ is the transition probability function, $r: \cS \times \cA \rightarrow [0, 1]$ is the reward function, $\rho \in \Delta_{\cS}$ is the initial state distribution, and $\gamma \in [0, 1)$ represents the discount factor. In this paper, we exclusively consider the setting where the number of states and actions is finite, but potentially large. 

Given $s \in \cS$, the policy $\pi$ induces a probability distribution $\ppi(. | s)$ over the actions. The action-value function $\qpi: \cS \times \cA \rightarrow \mathbb{R}$ induced by $\pi$ is defined as $\qpi(s, a) := \E[\sum_{t=0}^{\infty} \gamma^t r(s_t, a_t)| s_0= s , a_0 = a]$ where $s_t \sim p(. | s_{t-1}, a_{t-1})$, and $a_t \sim \ppi(.|s_t)$ for $t \ge 1$. The value function corresponding to $\qpi$ starting from state $s$ is defined as $V^{\pi}(s) := \E_{a \sim \ppi(. | s)}[\qpi(s, a)]$ with $J(\pi) := V^{\pi}(\rho) := \E_{s \sim \rho} [V^{\pi}(s)]$ representing the expected discounted cumulative reward. Furthermore, the advantage function $\api: \cS \times \cA \rightarrow \mathbb{R}$ is defined as $\api(s, a) := \qpi(s, a) - V^{\pi}(s)$. The policy also induces a discounted state-occupancy measure $\dpi(s) := (1 - \gamma) \sum_{t=0}^{\infty} \gamma^t \prpi[s_t = s | s_0 \sim \rho]$ over the states. The objective is to find an optimal policy $\pi^*$ that maximizes the expected reward $J(\pi) $, i.e. $\pi^* = \argmax_{\pi} J(\pi)$. As a special case, in the bandit setting, $|\cS| = 1$, $|\cA| = K$, $\gamma=0$, and $J(\pi)  = \langle \pi , r\rangle$, with $K$ representing the number of arms. 

%% file: algorithm.tex
\vspace{-2ex}
\section{SOFTMAX POLICY MIRROR ASCENT: TABULAR PARAMETRIZATION}
\label{sec:algorithm}
\vspace{-2ex}
Softmax policy mirror ascent (referred to as $\Alg$) represents the policy using the softmax function $h: \R^{A} \rightarrow \Delta_{A}$ i.e. $\pi(\cdot|s) = h(\z(s,\cdot))$ s.t. for all $(s,a) \in \cS \times \cA$, $\ppi(a|s)  = \frac{\exp(z(s,a))}{\sum_{a'}\exp(z(s, a'))}$, where the \textit{logits} $z$ are $SA$-dimensional vectors and $\Delta_{A}$ is the $A$-dimensional simplex. We first focus on the \textit{tabular parameterization} where the number of parameters scales with the number of states and actions, and aim to learn the logits corresponding to an optimal policy. With some abuse of notation, we use $J(\z)$ to refer to $J(\pi)$ where $\pi(\cdot|s) = h(z(s,\cdot))$ and state the objective as: $\max_{\z \in \R^{SA}} J(\z)$. 

As the name suggests, $\Alg$ uses mirror ascent (MA) to maximize $J(\z)$. For a differentiable, strictly convex mirror map $\Phi$, MA~\citep{beck2003mirror,bubeck2015convex} is an iterative algorithm whose update at iteration $t \in [T]$ can be stated in two equivalent ways:   
\begin{align}
\nabla \Phi(\ztt) &= \nabla \Phi(\zt) + \eta \, \nabla_{z} J(\zt) \label{eq:md-update} \\
\ztt & = \argmax_{\z \in \R^{SA}} \left[\inner{\z - \zt}{\nabla_\z J(\z_t)} - \frac{1}{\eta} \, D_\Phi(\z, \zt) \right] \nonumber 
\end{align}
where $\zt$ is the logit at iteration $t$, $\eta$ is the step-size and $D_\Phi(z,z') := \Phi(\z) - \Phi(\z') - \inner{\nabla \Phi(z')}{z - z'}$ is the Bregman divergence between logits $\z$ and $\z'$ induced by the mirror map $\Phi$. Hence, the MA update at iteration $t$ can be interpreted as moving in the gradient direction $\nabla_{z} J(\zt)$ while staying ``close'' to the logit $\zt$, where the proximity between logits is measured according to the Bregman divergence and weighted by $\frac{1}{\eta}$. 

\vspace{-2ex}
\subsection{Bandit Setting}
\label{sec:tab-mab}
\vspace{-2ex}
It is instructive to first instantiate the $\Alg$ update for the bandit setting where $J(\pi) = \langle \pi, r \rangle$. In this setting, $\nabla_{z} J(\z) \in \R^{A}$ s.t. $[\nabla_{z} J(\z)](a) = \pi(a) \, [r(a) - \inner{\pi}{r}]$. Following~\citet{vaswani2021general}, we use the log-sum-exp mirror map i.e. $\phi(z) =  \ln(\sum_{a} \exp(z(a))$. Since $[\nabla \phi(z)](a) = \frac{\exp(z(a))}{\sum_{a'} \exp(z(a'))} = [h(z)](a) = \pi(a)$, the logit and probability spaces are dual to each other, and the $\nabla \phi$ map can be used to move between these spaces. Given this, the $\Alg$ update can be written as:
\begin{align}
\ppitt(a) &= \ppit(a) \, (1 + \eta \, [r(a) - \inner{\pi}{r}] ) \nonumber \\
& = \ppit(a) \; [1 + \eta \sum_{a' \neq a} \ppit(a') \, \Delta(a, a')] \,,
\label{eq:tab-mab-spma}
\end{align}
where $\Delta(a, a') := r(a) - r(a')$ represents the reward gap between arms $a$ and $a'$. We first ensure that $\ppitt$ is a valid probability distribution. Since $r(a) \in [0,1]$ for all $a$, $\eta \le 1$ is sufficient to guarantee that $\ppitt(a)$ is non-negative for every $a$. Moreover, $\sum_{a} \ppitt(a) = \sum_{a} \ppit(a) + \eta \, \sum_{a} \ppit(a) [r(a) - \inner{\pi}{r}] = \sum_{a} \ppit(a) = 1$. Hence, for $\eta \leq 1$,~\cref{eq:tab-mab-spma} results in a valid update to the policy. The above update is related to the \texttt{PROD} algorithm~\citep{cesa2007improved} used for the experts problem in the online learning literature. In contrast to $\Alg$ which is derived from mirror ascent, \texttt{PROD} is derived using a linearization of the Hedge~\citep{freund1997decision} algorithm and requires explicit normalization to obtain probabilities. 
\vspace{-2ex}
\subsection{MDP Setting}
\label{sec:tab-mdp}
\vspace{-2ex}
In order to extend $\Alg$ to the MDP setting, we use a (state-wise) weighted log-sum-exp mirror map, i.e. for a logit $z \in \R^{SA}$, we define $\Phi(z):= \sum_{s} w(s) \, \phi(z(s,\cdot)) = \sum_{s} w(s) \ln(\sum_{a}\exp(z(s, a))$ where $w(s)$ are the per-state weights. Following the proof of~\citet[Lemma 11]{vaswani2024decision}, the resulting Bregman divergence is given as: $D_\Phi(z,z') = \sum_{s} w(s) \, \text{KL}(\pi'(\cdot|s) || \pi(\cdot|s))$ where $\pi$ and $\pi'$ are the policies corresponding to logits $z$ and $z'$. At iteration $t$ of $\Alg$, we choose $w(s) = d^{\ppit}(s)$ and use the policy gradient theorem~\citep{sutton1999policy} to calculate $[\nabla J(\zt)](s,a) = d^{\ppit}(s) \, \ppit(a|s) \, \apit(s,a)$. The resulting $\Alg$ update is given as:
\begin{align*}
\ztt = \argmax_{\z \in \R^{SA}} \sum_{s} d^{\ppit}(s) \bigg[& \langle \ppit(\cdot|s) \, \apit(s,\cdot), \, z(s,\cdot) \rangle \\ & - \frac{1}{\eta} \, \text{KL}(\ppit(\cdot|s) \, || \, h(z(s,\cdot)) \bigg] \,.
\end{align*}
Since the above maximization decomposes over the states, we can write the per-state update for each $s \in \cS$ in terms of $\ppitt(\cdot|s) = h(\ztt(s,\cdot))$ as follows:
\begin{align}
\ppitt(a | s) &= \ppit(a | s) \, (1 + \eta \apit(s, a)) \,.
\label{eq:tab-mdp-spma}
\end{align}
Similar to the bandit case, since $r(s, a) \in [0,1]$, $\ppitt(a|s)$  is non-negative for $\eta \le 1-\gamma$. Since $\sum_{a} \ppit(a | s) \apit(s, a) = 0$, $\sum \ppitt(a|s) = 1$, and hence~\cref{eq:tab-mdp-spma} results in a valid policy update. 

In order to compare the $\Alg$ update to existing methods, note that for the tabular parameterization, natural policy gradient (\texttt{NPG}) update~\citep{kakade2001natural} is the same as policy mirror ascent~\citep{lan2023policy,johnson2023optimal,xiao2022convergence} and is given by: $\ppitt(a|s) \propto \ppit(a|s) \, \exp(\eta \, \apit(s,a))$. In contrast to $\NPG$, the $\Alg$ update in~\cref{eq:tab-mdp-spma} is linear in both $\eta$ and $\apit(s,a)$ and does not require an explicit normalization across actions to ensure valid probability distributions. On the other hand, softmax policy gradient (\texttt{SPG})~\citep{agarwal2021theory,mei2020global} corresponds to choosing the mirror map $\phi$ in~\cref{eq:md-update} to be the Euclidean norm and has the following update: $\ztt(s,a) = \zt(s,a) + \eta \, \ppit(a|s) \apit(s,a)$. Compared to \texttt{SPG} that uses the softmax policy gradient to update the logits, $\Alg$ uses the softmax policy gradient to directly update the probabilities. As we demonstrate in the next section, this desirable property enables $\Alg$ to achieve faster rates than \texttt{SPG}.

%% file: theory.tex
\vspace{-2ex}
\subsection{Theoretical Results}
\label{sec:tab-theory}
\vspace{-2ex}
In this section, we prove convergence guarantees for $\Alg$ in the multi-armed bandit and tabular MDP settings. We first establish linear convergence for $\Alg$ for multi-armed bandits for any constant $\eta \le 1$.
\begin{restatable}{theorem}{smdpoconstbandit}
\label{constant-eta-bandit}
The $\Alg$ update in~\cref{eq:tab-mab-spma} with (i) a constant step-size $\eta \le 1$, and (ii) uniform initialization i.e. $\ppi_{0}(a) = \frac{1}{K}$ for all $a$ converges as:
\begin{align*}
r(a^*) - \langle \pi_T, r \rangle \leq \left(1 - \frac{1}{K} \right) \exp\left(\frac{-\eta \, \Delta_{\min} \, T}{K}\right) \,,
\end{align*}
where $T$ is the number of iterations, $a^*$ is the optimal arm i.e. $a^* = \argmax_{a} r(a)$ and $\Delta_{\min} := \min_{a \neq a^*} \Delta(a^*, a) = r(a^*) - r(a)$ is the gap. 
\label{thm:tab-spma-bandit-linear}
\end{restatable}
The above theorem (proved in~\cref{app:multi-armed-bandit-proofs}) shows that for multi-armed bandit problems, $\Alg$ can achieve linear convergence to the optimal arm, and the resulting rate depends on both the gap and the number of arms. In~\cref{app:tab-mab-spma-super}, we prove that $\Alg$ with specific gap-dependent step-sizes can achieve a global super-linear convergence rate for multi-armed bandits. To the best of our knowledge, these are the first global super-linear rates for PG methods on multi-armed bandit problems.   

In the next theorem, we extend the linear convergence result to tabular MDPs and prove that when given access to the exact policy gradients, $\Alg$ results in linear convergence to the optimal value function for any sufficiently small constant step-size. 
\begin{restatable}{theorem}{smdpoconstmdp}
Using the $\Alg$ update in~\cref{eq:tab-mdp-spma} with a step-size $\eta < \min\left\{1 - \gamma, \frac{1}{C_t (1-\gamma)} \right\}$ converges as:
\begin{align*}
& \norminf{V^{\pi^{*}} - V^{\pi_{T}}} \le \left(\prod_{t=0}^{T-1} \alpha_t \right)  \norminf{V^{\pi^{*}} - V^{\pi_{0}}} \,,
\end{align*}
where $\alpha_t := (1 - \eta \, C_t \, (1- \gamma) )$, $C_t := \min_s \{ \ppit(\tilde{a}_t(s) | s)  \, \Delta^t(s) \}$, $\tilde{a}_t(s) := \argmax_a \qpit(s, a)$ and $\Delta^t(s) := \max_a \qpit(s, a) - \max_{a\ne \tilde{a}} \qpit(s, a)$.
\label{thm:smdpo-constant-mdp}
\end{restatable}
For ease of exposition, the above theorem considers $\tilde{a}_t(s)$ to be the unique action maximizing $\qpit(s, \cdot)$ for every state $s$ and policy $\ppit$. In~\cref{app:mdp-proof}, we extend this to include multiple optimal actions with a minor change in the definition of the gap. For rewards in $(0,1)$, $C_t \, (1-\gamma)$ is in $(0,1)$ and depends on the initialization. If $C := \min_{t \in [T]} C_t$, then the above implies that when $T \in O\left(\frac{1}{\eta \, C (1 - \gamma)} \, \ln(\nicefrac{1}{\epsilon})\right)$, $\Alg$ guarantees that $V^{\pi_T}(s) \geq V^{*}(s) - \epsilon$ for all $s \in \cS$.   
Note that in order to establish linear convergence, it is crucial to ensure that $C$ is nonzero. Proving this property theoretically is challenging, as reflected in related work. In particular, the best-known linear convergence rates for $\NPG$ with a constant step size also depend on this same constant (see, for example, Theorem~5.4 in~\citep{liu2024elementary} and Lemma~10 in~\citep{mei2021understanding}). Consequently, in~\cref{app:tabular-mdp-exp-ct}, we empirically verify that $C$ is lower-bounded by a positive constant. 

In order to put the above convergence result in context, note that \texttt{SPG} with a constant step-size results in a $\Theta\left(\nicefrac{1}{\epsilon}\right)$ convergence~\citep{mei2020global}. Recently,~\citet{chenaccelerated} proved that constant step-size \texttt{SPG} with Nesterov acceleration can obtain an $O(1/\sqrt{\epsilon})$ convergence. In contrast, the above theorem demonstrates that by choosing the appropriate mirror map, constant step-size $\Alg$ can achieve a faster $O(\log(1/\epsilon))$ rate of convergence. On the other hand,~\citet{liu2024elementary,lu2024towards} prove that \texttt{SPG} with adaptive step-sizes can also result in linear convergence. However, the resulting rate depends on the distribution mismatch ratio $\left\|\frac{d^{\pi^*}}{\rho}\right\|_{\infty}$ that can be exponentially large in the size of the state space~\citep{li2021softmax}. In contrast, the convergence result in~\cref{thm:smdpo-constant-mdp} has no such dependence. The linear convergence rate in~\cref{thm:smdpo-constant-mdp} matches that of \texttt{NPG} with a constant step-size~\citep{liu2024elementary} and compared to~\citet[Theorem 5.4]{liu2024elementary}, it results in a better dependence (exponential vs polynomial) on the gap $\Delta^t(s)$. Finally, we note that for the tabular parameterization, a variant of $\TRPO$ has been shown to achieve $O(1/\epsilon^2)$ convergence to the optimal policy~\citep{shani2020adaptive}.

In the next section, we extend $\Alg$ to exploit function approximation to handle large state-action spaces.

%% file: fa.tex
\vspace{-2ex}
\section{HANDLING FUNCTION APPROXIMATION}
\label{sec:fa}
\vspace{-2ex}
Handling large MDPs requires function approximation (FA) techniques to share information between states and actions. For example, given a set of state-action features $\X \in \R^{SA \times d}$ where $d << SA$, the \textit{log-linear policy parameterization}~\citep{agarwal2021theory,alfano2022linear,yuan2023linear} considers policies of the form: $\pi(a|s) = \frac{\exp(\langle \X(s,a), \theta \rangle)}{\sum_{a'} \exp(\langle \X(s,a'), \theta \rangle)}$ where $\theta \in \R^{d}$ is the parameter to be learned. Hence, the log-linear policy parameterization can handle large state-action spaces while learning a compressed $d$-dimensional representation. 
\begin{algorithm}[!h]
\caption{$\Alg$ with function approximation}
\label{alg:spma}
\textbf{Input}: $\theta_0$ (parameters for the initial policy $\pi_0$), $f_{\theta}$ (function approximation), $T$ (number of outer-loop), $m$ (number of inner-loops), $\eta$ (outer-loop step-size), $\zeta$ (inner-loop step-size) \\
\For{$t \leftarrow 0$  \KwTo  $T-1$}{

    \texttt{1.} Interact with the environment using $\ppit$ and form the surrogate function $\ell_t(\theta)$ in~\cref{eq:approx-surrogate} 

    \texttt{2.} Initialize inner-loop: $\omega_0 = \theta_{t}$ \\
    
    \For{$k \leftarrow 0$   \KwTo  $m-1$}{
        $\omega_{k+1} = \omega_{k} - \zeta \, \nabla_{\omega} \ell_t(\omega_k)$ 
    }
    \texttt{3.} $\theta_{t+1} = \omega_{m}$ \\
    
    \texttt{4.} Update $\ppitt(\cdot|s) = h(f_{\theta_{t+1}}(s,.))$
}
Return $\theta_{T}$
\end{algorithm}
We interpret the log-linear policy parameterization as a constraint in the space of logits. Specifically, the logits $\z$ are constrained to lie in the set $\cZ = \{z \in \R^{SA} \vert \exists \theta \text{ s.t. } z = \X \theta$ \}, meaning that the logits are required to be realizable by the linear model with features $\X$. We define $\Pi$ as the corresponding set of feasible policies, i.e. $\Pi = \{\pi \vert \forall s \in \cS \,, \pi(\cdot|s) = h(z(s,\cdot)) \text{ s.t. } \z \in \cZ \}$. Hence, the policies in $\Pi$ are constrained to be log-linear. Note that, as in the case of log-linear policies, $\Pi$ can be a non-convex set, even when $\cZ$ is convex. For general \textit{energy-based or neural policies}~\citep{haarnoja2017reinforcement,agarwal2021theory}, $\pi(a|s) \propto \exp(f_{\theta}(s,a))$ where $f_{\theta}: \R^{SA} \rightarrow \R$ is a complex, non-linear model. In this case, the logits are constrained to lie in the set: $\cZ = \{z \in \R^{SA} \vert \exists \theta \text{ s.t. } z(s,a) = f_{\theta}(s,a) \}$. 

The above interpretation allows us to extend $\Alg$ to the FA setting. Specifically, we use the same mirror ascent update as in~\cref{eq:md-update} with an additional projection step onto the feasible set $\cZ$. Specifically, we define $\zth$ s.t. $\nabla \Phi(\zth) = \nabla \Phi(\zt) + \eta \, \nabla_{z} J(\zt)$ and compute $\ztt = \argmin_{z \in \cZ} D_\Phi(\z, \zth)$. This step denotes the Bregman projection of $\zth$ onto $\cZ$, i.e. we seek to find the closest (according to the Bregman divergence) realizable point (in $\cZ$) to the ``ideal'' point $\zth$ which corresponds to using the tabular parameterization. Following~\citet{vaswani2021general,lavington2023target}, we convert the above projection problem into an unconstrained minimization problem where $\forall (s,a)$, $\ztt(s,a) = f_{\thtt}(s,a)$, $z_{\theta}(s,a) := f_{\theta}(s,a) \in \cZ$, $\thtt = \argmin_{\theta} D_\Phi(z_{\theta}, \zth)$ i.e. we aim to find the parameter $\theta$ that realizes the point $\z_\theta \in \cZ$  which is closest to $\zth$. Following~\cref{sec:tab-mdp}, using the log-sum-exp mirror map weighted by $d^{\ppit}(s)$ at iteration $t$ results in the following optimization problem $\thtt = \argmin_{\theta} \tilde{\ell}_t(\theta)$ where,  
\begin{align}
& \tilde{\ell}_t(\theta) := \sum_{s} d^{\ppit}(s) \, \text{KL}(\ppith(\cdot|s) \, || \, \pi_{\theta}(\cdot|s)) \label{eq:spma-surrogate-main} \\
&= \underset{s \sim d^{\ppit}}{\E}  \, \cH \left(h(f_{\tht}(s,\cdot))  (1 + \eta  A^{\ppit}(s,\cdot)), h(f_{\theta}(s,\cdot))  \right) + C_t \nonumber \,.
\end{align}
Here, $\cH(p,q) := -\E_{p}[\ln(q)] = - \sum_{a} p(a) \ln(q(a))$ is the cross-entropy between distributions $p$ and $q$ and $C_t$ is a constant with respect to $\theta$. We refer to $\tilde{\ell}_t(\theta)$ as the \textit{ideal surrogate}. Minimizing this surrogate requires calculating the expectation over the states sampled according to $\ppit$. In order to have a practical algorithm, we can run trajectories $\tau$ starting from the initial state distribution $\rho$, following the policy $\ppit$ and thus sampling from the $d^{\ppit}$ distribution (see~\citet[Algorithm 3]{agarwal2021theory} for the detailed procedure). Given these samples, we form the surrogate $\ell_t(\theta)$ defined as:  
\begin{align}
\sum_{s \sim \tau} \text{KL}\left(h(f_{\tht}(s,\cdot))  (1 + \eta  A^{\ppit}(s,\cdot))\, || \, h(f_{\theta}(s,\cdot))  \right) \,.
\label{eq:approx-surrogate}    
\end{align}
Note that $\E[\ell_t(\theta)] = \tilde{\ell}_t(\theta)$ where the expectation is w.r.t. to $d^{\ppit}$. We use $m$ steps of (stochastic) gradient descent to approximately minimize $\ell_t(\theta)$. Putting everything together, the algorithm incorporating general FA is presented in~\cref{alg:spma}. 

\textbf{Log-linear Policy Parameterization:} For this special case, the problem in~\cref{eq:spma-surrogate-main} is equivalent to a weighted (according to $d^{\ppit}(s)$) multi-class classification for each state. The per-state problem corresponds to a softmax classification into $A$ classes using a linear model with features $\X$ and soft labels equal to $\ppith(\cdot|s)$. Computing $\thtt$ thus involves minimizing a smooth, convex function. 

In the next section, we compare~\cref{alg:spma} to existing approaches that incorporate FA. 

\vspace{-2ex}
\subsection{Comparison to Existing Approaches}
\label{sec:fa-comparison}
\vspace{-1ex}
\textbf{Comparison to $\NPG$}: A principled extension of $\NPG$ to handle FA is via the compatible function approximation~\citep{kakade2001natural,agarwal2021theory}. An example of such an algorithm, \texttt{Q-NPG} involves solving a quadratic surrogate at each iteration $t$: $\hat{\omega}_{t} = \min_{\omega} \sum_{s} d^{\ppit}(s) \sum_{a} \ppit(a|s) \left(f_{\omega}(s,a) - Q^{\ppit}(s,a) \right)^2$. The policy parameters are updated using $\hat{\omega}_t$ which corresponds to the natural gradient direction. While this approach results in theoretical guarantees (see~\cref{sec:fa-theory} for details); for a general parameterization, the resulting algorithm involves changing the representation of the critic at every iteration. Consequently, solving the surrogate is expensive, limiting the practicality of the method.

\textbf{Comparison to $\MDPO$}: A more practical extension of $\NPG$ is mirror descent policy optimization~\citep{tomar2020mirror} ($\MDPO$). Similar to $\Alg$, $\MDPO$ can be interpreted as projected (onto the feasible set of policies) mirror ascent in the space of probabilities~\citep{vaswani2021general}. The resulting surrogate (as a function of the policy parameters) is given by: $\sum_{s} d^{\ppit}(s) \, \text{KL}\left(\pi_{\theta}(\cdot|s) \, || \, h\left(f_{\tht}(s,\cdot) \, \exp(\eta \, Q^{\ppit}(s,\cdot) ) \right)  \right)$. Unlike the surrogate in~\cref{eq:spma-surrogate-main}, the  $\MDPO$ surrogate is non-convex even when using a tabular softmax parameterization for the policy, and consequently does not have any theoretical guarantees. However, $\MDPO$ results in good empirical performance, and we compare to it in~\cref{sec:experiments}. 

\textbf{Comparison to \texttt{TRPO}}: As explained in~\citet{vaswani2021general}, the surrogate in~\cref{eq:spma-surrogate-main} is closely related to $\texttt{TRPO}$. In particular, the $\texttt{TRPO}$ update consists of solving the following optimization problem: $\sum_{s} d^{\ppit}(s) \sum_{a} \ppit(a|s)A^{\ppit}(s, a) \, \frac{\pi_{\theta}(a|s)}{\pi_{\tht}(a|s)}$, such that $\E_{s \sim d^{\ppit}} \left[\text{KL}(\ppit(\cdot|s) \, || \, \pi_{\theta}(\cdot|s)) \right] \leq \delta$. $\Alg$ (i) uses instead $\sum_{s} d^{\ppit}(s) \sum_{a} \ppit(a|s) A^{\ppit}(s, a) \, \log \frac{\pi_{\theta}(a|s)}{\pi_{\tht}(a|s)}$, i.e. the logarithm of the importance sampling ratio, making the resulting update more stable~\citep{vaswani2021general} and (ii) enforces the proximity between policies via a regularization rather than a constraint. Enforcing the trust-region constraint in $\texttt{TRPO}$ requires additional hyper-parameters, code-level optimizations and computation~\citep{engstrom2019implementation}. In contrast, $\Alg$ is more computationally efficient and simpler to implement in practice. In the next section, we study the theoretical properties of~\cref{alg:spma}. 
\vspace{-2ex}
\subsection{Theoretical Guarantee}
\label{sec:fa-theory}
\vspace{-2ex}
For rewards in $[0,1]$ and for a general policy parameterization,~\citet{vaswani2021general} prove that, for $\eta \leq 1 - \gamma$,~\cref{alg:spma} results in monotonic improvement, i.e. $J(\ppitt) \geq J(\ppit)$ and hence converges to a stationary point at an $O(1/\epsilon)$ rate. Since $J$ is non-convex and can have multiple stationary points, the result in~\citet{vaswani2021general} does not provide sufficient evidence for the good empirical performance of~\cref{alg:spma}. In this section, we prove that, under reasonable assumptions similar to existing works,~\cref{alg:spma} can converge to the neighbourhood of the optimal value function at a linear rate. The size of the neighbourhood is determined by various practical factors such as sampling, inexact optimization, and bias due to the FA. In order to state our result, we first state and justify our assumptions. 

Recall that in order to have a practical algorithm, we minimize $\ell_t(\theta)$ obtained by sampling from $d^{\ppit}$.   
\begin{assumption}
\textit{Excess Risk}: For all iterations $t$ of~\cref{alg:spma}, $\vert \tilde{\ell}_{t}(\theta_{t+1}) - \min \tilde{\ell}_{t}(\theta) \vert \leq \epsilon_{\text{stat}}$.  
\label{as:samplingerror}     
\end{assumption}
\vspace{-1ex}
The above assumption quantifies the excess risk incurred by minimizing a finite sampled ``dataset'' of states as compared to minimizing over the population loss $\tilde{\ell}_t(\theta)$. This is a standard assumption in the literature analyzing the convergence of policy gradient methods with FA~\citep{agarwal2021theory,alfano2022linear,yuan2023linear}. If $n$ is the number of samples and the surrogate is minimized using (stochastic) gradient descent, using the standard generalization results~\citep{lei2021sharper,nikolakakis2022beyond}, we expect $\epsilon_{\text{stat}} = O(1/n)$ for the log-linear parameterization. For example, using $m$ iterations of SGD would result in $\epsilon_{\text{stat}} = O(1/n + 1/m)$~\citep[Theorem 6]{lei2021sharper}. For a general parameterization, where the surrogate might be non-convex, the excess risk can be bounded up to the optimization error~\citep{nikolakakis2022beyond}. Under the appropriate technical assumptions, $\ell_t(\theta)$ can been shown to satisfy the Polyak-Lojasiewicz condition ~\citep{liu2022loss} implying that the optimization error for (stochastic) gradient descent can be made arbitrarily small. The next assumption quantifies the bias incurred because of a policy parameterization with limited expressive power compared to using the tabular parameterization. 
\vspace{-2ex}
\begin{assumption}
\textit{Bias}: For all iterations $t$ of~\cref{alg:spma}, $\min_{\theta} \tilde{\ell}_t(\theta) \leq \epsilon_{\text{bias}}$.
\label{as:bias}
\end{assumption}
\vspace{-1ex}
The above assumption captures the flexibility of the model class being used in the policy parameterization. For a tabular parameterization where the number of parameters scales as $SA$, $\epsilon_{\text{bias}} = 0$, whereas for the log-linear parameterization, $\epsilon_{\text{bias}}$ depends on the expressivity of the features. The final assumption is concerned with exploration and indicates that the initial state distribution has full support implying that the method does not require explicit exploration. 
\begin{assumption}
\textit{Exploration}: $\forall s \in \cS$, $\rho(s) \geq \rho_{\min} > 0$. 
\label{as:exploration}
\end{assumption}
\vspace{-4ex}
The above assumption is standard in the literature~\citep{agarwal2021theory,xiao2022convergence} and helps isolate and study the optimization properties of PG methods. We prove the following theorem in~\cref{app:mdp-proof}. 
\begin{restatable}{theorem}{smdpofuncapp}
Under~\cref{as:samplingerror}-\ref{as:exploration}, $\Alg$ with $\eta < \min\left\{1 - \gamma, \frac{1}{C_t (1-\gamma)} \right\}$ converges as,
\vspace{-2ex}
\begin{align*}
& J(\pi^*) - J(\pi_T) \\ & \leq \left(\prod_{t=0}^{T-1} \alpha_t\right) (J(\pi^*) - J(\pi_0)) + \beta \sum_{t=0}^{T-1}\prod_{i=t+1}^{T-1}\alpha_i \,,
\end{align*}
where $\beta = \frac{\sqrt{2}}{(1-\gamma)^2\rho_{\min}} \sqrt{\epsilon_{\text{stat}}  + \epsilon_{\text{bias}}}$ and $\alpha_t$ has the same definition as in~\cref{thm:smdpo-constant-mdp}.
\label{thm:spma-fa}
\end{restatable}
\vspace{-2ex}
The above theorem shows that~\cref{alg:spma} converges linearly to the neighbourhood of the optimal value function. Furthermore, for the log-linear parameterization, the size of the neighbourhood can be bounded explicitly. For example, if the logits $z_{t+1/2}$ for every $t$ lie in the span of the features, $\epsilon_{\text{bias}} = 0$ (this is similar to the linear Bellman completeness condition used in the analysis of value-based methods~\citep{munos2005error}) and $\epsilon_{\text{stat}} = O(1/n + 1/m)$. By using more samples and with more (S)GD iterations, the size of the neighbourhood can be made arbitrarily small. Except for the neighbourhood term, the above convergence result is similar to that for the tabular setting in~\cref{thm:smdpo-constant-mdp}. The other difference is that the result in the tabular setting holds in the $\ell_\infty$ norm and thus holds for all states, whereas the result in~\cref{thm:spma-fa} only holds for a fixed starting state distribution $\rho$. In practice, $\apit$ is typically estimated via a critic. To account for this, we generalize the proof of~\cref{thm:spma-fa} in~\cref{app:mdp-proof}, and prove that~\cref{alg:spma} converges linearly to a neighbourhood that depends on an additional term proportional to the critic error.   

We now compare to the existing theoretical results for PG methods with FA. For the log-linear policy parameterization, \texttt{Q-NPG} and its variants have been shown to achieve linear convergence to the neighbourhood of the optimal value function~\citep{agarwal2021theory,alfano2022linear,yuan2023linear}. The size of the neighbourhood depends on similar quantities as ~\cref{thm:spma-fa}. Finally, we note that while an ideal, impractical variant of \texttt{TRPO} has a monotonic improvement guarantee similar to~\cref{alg:spma}~\citep{schulman2015trust}, it does not have convergence guarantees comparable to~\cref{thm:spma-fa}.


%% file: experiments.tex
\vspace{-2ex}
\section{EMPIRICAL EVALUATION}
\label{sec:experiments}
\vspace{-2ex}
\begin{figure*}[!ht]
\centering
\includegraphics[width=0.91\textwidth]{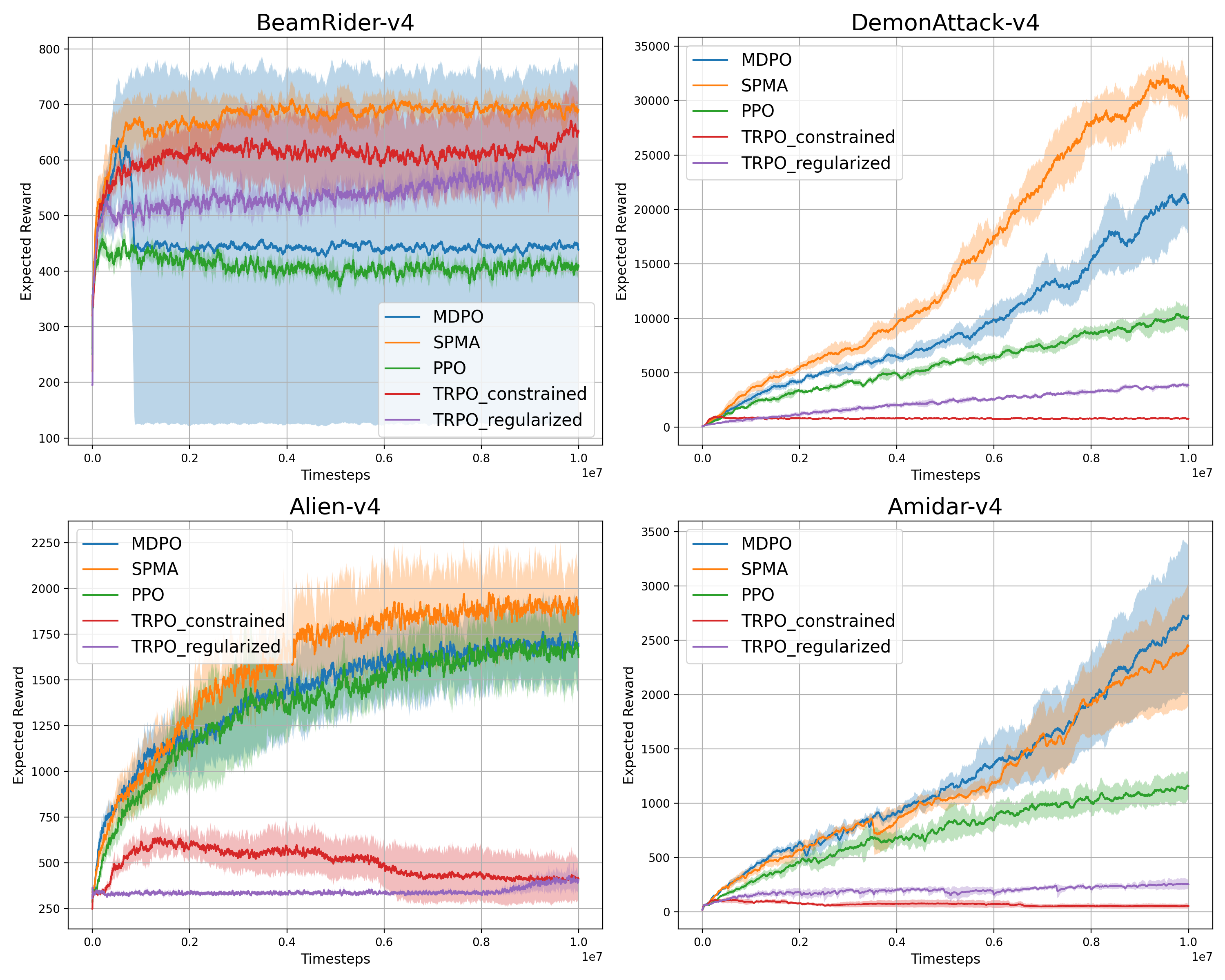}
\caption{On Atari games, where a CNN-based actor network is employed, $\Alg$ matches or surpasses $\MDPO$ and outperforms $\PPO$, as well as both the constrained and regularized versions of $\TRPO$.}
\label{fig:mdp-atari}
\end{figure*}

\begin{figure*}[!ht]
\centering
\includegraphics[width=0.91\textwidth]{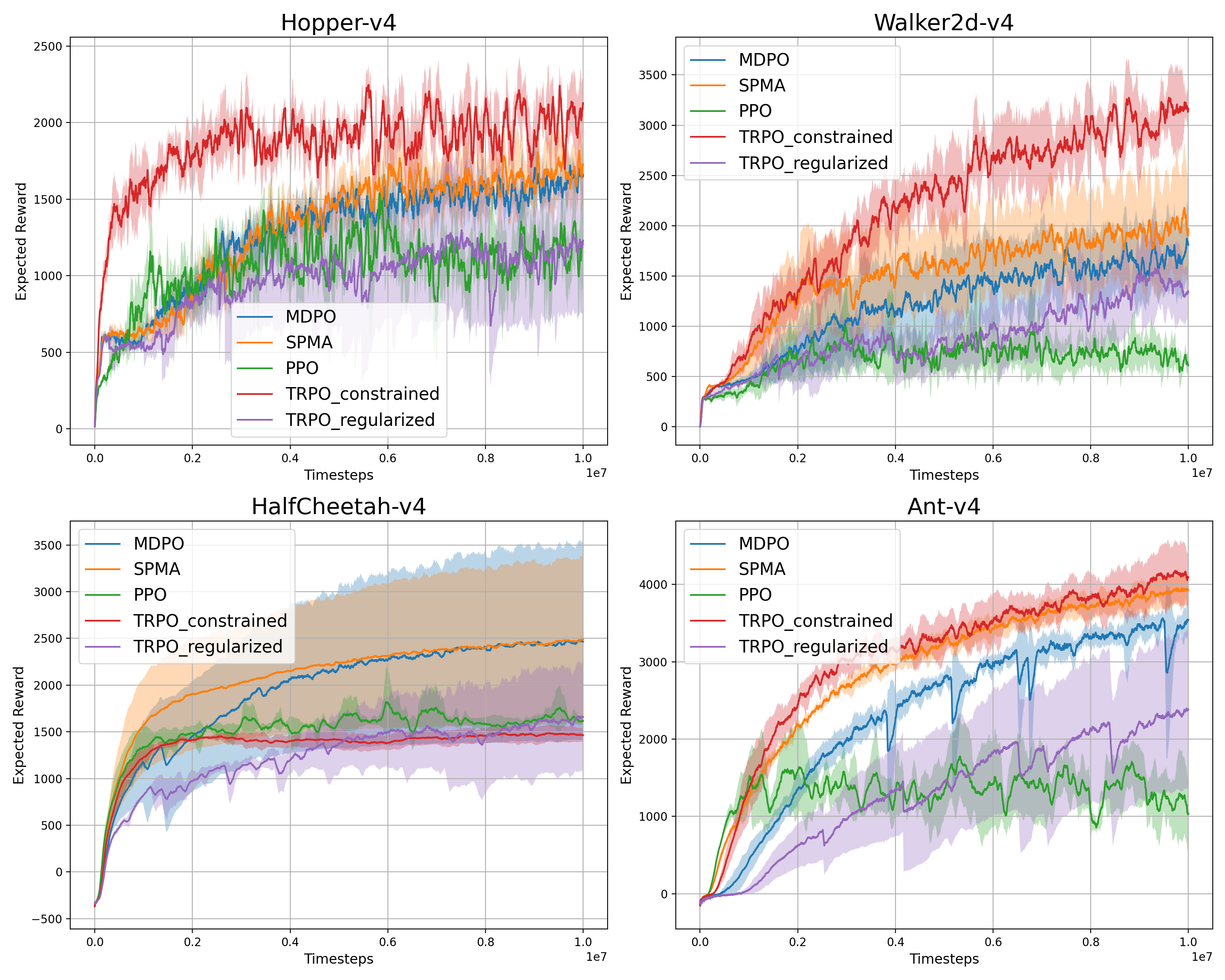}
\caption{On MuJoCo control tasks, where a two-layer MLP actor network is used, $\Alg$ matches or outperforms $\MDPO$ while consistently outperforming $\PPO$ and regularized $\TRPO$. In contrast to the results on Atari games, with a shallow MLP, $\TRPOC$ outperforms all methods.}
\label{fig:mdp-mujoco}
\end{figure*}

We evaluate $\Alg$\footnote{The code is available at 
\href{https://github.com/reza-asad/SPMA/tree/main}{\tt http://github.com/\newline reza-asad/SPMA}.} on three types of problems: (i) tabular MDPs with access to exact policy gradients, (ii) MDPs with continuous states but discrete actions, using inexact policy gradients, and (iii) MDPs with continuous state-actions spaces and inexact gradients. For tabular MDPs, we use the tabular parameterization and compare $\Alg$ against $\NPG$ and constant step-size $\SPG$~\citep{mei2020global}. For these environments, we also consider log-linear policies and compare $\Alg$ to $\MDPO$ and $\SPG$. For non-tabular environments, we consider $\PPO$, $\TRPO$ and $\MDPO$ as baselines. We consider two variants of $\TRPO$ -- $\TRPOC$, the standard optimized variant in~\citet{stable-baselines3} and $\TRPOR$, the original regularized variant. $\TRPOC$ is able to effectively enforce the trust region constraint using conjugate gradient, but introduces additional hyper-parameters, requires code-level optimization techniques and is computationally expensive. On the other hand, $\TRPOR$ is significantly more efficient and theoretically principled~\citep{lazic2021optimization}, and is similar to $\Alg$'s objective (see~\cref{sec:fa-comparison}). For details regarding the hyper-parameters of all methods for each environment, refer to~\cref{app:tabular-mdp-experiments,app:stable-baseline-experiment-details}. 




\textbf{Tabular MDP Results:} We present the results in~\cref{app:tabular-mdp-experiments}, and summarize the key findings here. We observe that $\Alg$ and $\NPG$ achieve comparable performance, both consistently outperforming $\SPG$ (\cref{fig:mdp-tabular}). However, in the linear FA setting, $\Alg$ demonstrates superior performance compared to $\MDPO$ in the CliffWorld environment~\citep{sutton2018reinforcement} (\cref{fig:cw-m-25-50-lfa}), while performing similarly in the Frozen Lake environment~\citep{brockman2016openai} (\cref{fig:fl-m-25-50-lfa}). In both environments, $\Alg$ and $\MDPO$ consistently outperform $\SPG$.

In the remainder of this section, we focus on the non-tabular settings with inexact policy gradients. For these experiments, we follow the protocol of \citet{tomar2020mirror}, using 5 seeds and reporting the average results along with their 95\% confidence intervals. Additionally, we employ the actor-critic architecture, policy parameterization, and GAE~\citep{schulman2015high} (to estimate the advantage function) from stable baselines~\citep{stable-baselines3}. We emphasize that, in contrast to prior work, we do not make ad-hoc adjustments to $\Alg$ (i.e., the actor). To set the step-size $\eta$ in~\cref{alg:spma}, we perform a grid search over five values (fixed across all experiments) and set the inner loop step-size $\zeta$ using Armijo line search~\citep{armijo1966minimization}.

\textbf{Atari and Mujoco Results:} We evaluate the performance of $\Alg$ compared to the baselines across various Atari 2600 games~\citep{bellemare2013arcade} and MuJoCo~\citep{todorov2012mujoco} control tasks from OpenAI Gym~\citep{brockman2016openai}. The observation space for Atari games consists of a 210$\times$160$\times$3 image, representing the current state of the game. The action space in these environments is discrete, whereas in MuJoCo, it is continuous and by default represented by a diagonal Gaussian distribution in~\citet{stable-baselines3}. Additionally, the actor-critic network for Atari uses a CNN as a feature extractor, while MuJoCo employs an MLP. 



Comparing the results in Fig.~\ref{fig:mdp-atari} and~\ref{fig:mdp-mujoco}, our key observations are as follows: i) $\Alg$ consistently outperforms or matches $\MDPO$ and $\PPO$ across all environments; ii) although $\TRPOC$ achieves superior performance on MuJoCo, its performance degrades considerably on Atari games. We conjecture that the conjugate gradient algorithm in $\TRPOC$ performs poorly when the actor network is a CNN rather than a two-layer MLP; iii) $\TRPOR$, which has a similar objective as $\Alg$ (see~\cref{sec:fa-comparison} for a comparison) does not perform as well on MuJoCo and has considerably worse performance on Atari. Hence, we observe that replacing the sampling ratio by its $\log$ can result in substantial empirical gains. This behaviour has also been observed for $\PPO$~\citet{vaswani2021general}. Overall, our experiments demonstrate that, despite being theoretically grounded, $\Alg$ exhibits strong empirical performance across various environments without relying on ad-hoc adjustments.

%% file: conclusion.tex
\vspace{-2ex}
\section{DISCUSSION}
\label{sec:discussion}
\vspace{-2ex}
We developed $\Alg$, a PG method that corresponds to mirror ascent in the dual space of logits. We believe that our paper bridges the gap between theoretical PG methods and practical objectives by presenting a method that offers strong theoretical convergence guarantees while delivering competitive practical performance (compared to $\PPO$, $\TRPO$, $\MDPO$), without relying on additional heuristics or algorithmic modifications. In the future, we aim to develop techniques for adaptively tuning the step-size and avoiding expensive grid-searches. We also plan to develop and analyze an off-policy variant of $\Alg$.

%% file: acknowledgement.tex
\section*{ACKNOWLEDGEMENTS}
\label{sec:ack}
This research was partially supported by the Natural Sciences and Engineering Research Council of Canada (NSERC) Discovery Grant RGPIN-2022-04816. 

%% file: appendix-setup.tex
\appendix
\newcommand{\appendixTitle}{%
\vbox{
    \centering
	\hrule height 4pt
	\vskip 0.2in
	{\LARGE \bf Supplementary Material}
	\vskip 0.2in
	\hrule height 1pt 
}}
\appendixTitle
\section*{Organization of the Appendix}\label{appendix:org}
\begin{itemize}
   \item[\ref{app:multi-armed-bandit-proofs}] \hyperref[app:multi-armed-bandit-proofs]{Multi-armed Bandit Proofs}
   
   \item[\ref{app:mdp-proof}] \hyperref[app:mdp-proof]{MDP Proofs} 

    \item[\ref{app:tabular-mdp-experiments}] \hyperref[app:tabular-mdp-experiments]{Tabular MDP Experiments} 

    \item[\ref{app:stable-baseline-experiment-details}] \hyperref[app:stable-baseline-experiment-details]{Additional Details for Stable Baselines Experiments}

\end{itemize}

%% file: appendix-proofs.tex
\section{Multi-armed Bandit Proofs}
\label{app:multi-armed-bandit-proofs}
\smdpoconstbandit*
\begin{proof}. As in equation $\eqref{eq:tab-mab-spma}$, we can write the update for arm $a$ as following where $\Delta(a,a')= r(a) - r(a')$,  
    \begin{align*}
        \ppitt(a) &= \ppit(a)\left[ 1 + \eta \sum_{a' \ne a} \ppit(a') \Delta(a, a')\right]         
    \end{align*}
    \begin{align}
        1 - \ppitt(a^*) &= 1 - \ppit(a^*) - \eta \, \ppit(a^*)\left[ \sum_{a' \ne a^*} \ppit(a') \Delta(a^*, a') \right] \label{prob-bad-action}
    \end{align}
    We first find a lower-bound for $\sum_{a' \ne a^*} \ppit(a') \Delta(a^*, a')$:
    \begin{equation}
    \begin{aligned}
        \sum_{a' \ne a^*} \ppit(a') \Delta(a^*, a') &\ge \Delta_{\min} \sum_{a' \ne a^*} \ppit(a')\\
            &= \Delta_{\min} (1 - \ppit(a^*)) \label{pi-delta-lower-bound}
    \end{aligned}
    \end{equation}
    Next, we observe that $\sum_{a' \ne a^*} \ppit(a') \Delta(a^*, a') \ge 0$. Using this information and starting with a uniform initialization for selecting an arm implies a monotonic improvement on the probability of selecting the optimal arm:
    \begin{align}
        \ppitt(a^*) > \ppit(a^*) > ... > \pi_0(a^*) = \frac{1}{K} \label{monotonic-prob-good-arm}
    \end{align}
    Let $\epsilon_t = 1 - \ppit(a^*)$. 
    \begin{align*}
        \epsilon_{t+1} &= \epsilon_t  - \eta \, \ppit(a^*)\left[ \sum_{a' \ne a^*} \ppit(a') \Delta(a^*, a') \right]\\
                         & \leq  \epsilon_t - \frac{\eta}{K} \, \left[ \sum_{a' \ne a^*} \ppit(a') \Delta(a^*, a') \right]\tag{using $\eqref{monotonic-prob-good-arm}$}\\  
                        &\le \epsilon_t -\frac{\eta \Delta_{\min}}{K}\epsilon_t\tag{using \eqref{pi-delta-lower-bound}}\\
                       &= \epsilon_t \, \left(1-\frac{\eta \Delta_{\min}}{K} \right)\\
    \end{align*}
    Recursing from $t=0$ to $t=T-1$ we have:
    \begin{align*}
        \epsilon_{T} &\le \epsilon_0 \, \left(1-\frac{\eta \Delta_{\min}}{K} \right)^T\\
                     &\le \epsilon_0 \exp\left(\frac{-\eta \Delta_{\min} T}{K}\right) \tag{using $1-x \le \exp(-x)$}\\
                     &= \left(1 - \frac{1}{K}\right) \exp\left(\frac{-\eta \Delta_{\min} T}{K}\right)
    \end{align*}
    Finally, we define the sub-optimality gap, $\delta_T:= r(a^*) - \langle \pi_T, r \rangle$:
    \begin{align*}
        \delta_T &= \sum_{a'} \pi_T(a') \left[r(a^*) - r(a')\right]\\
             &= \sum_{a' \ne a^*} \pi_T(a)\Delta(a^*, a)\\
             &\leq \max_{a'} \Delta(a^*, a') \sum_{a' \ne a^*} \pi_T(a)\\
             & = \max_{a'} \Delta(a^*, a') (1-\pi_{T}(a^*))\\
             &\le 1-\pi_{T}(a^*) \tag{using the fact $0\le r\le1$}\\
                 & = \epsilon_T \\
                 &\leq \left(1 - \frac{1}{K}\right) \exp\left(\frac{-\eta \Delta_{\min} T}{K}\right)
    \end{align*}
\end{proof}

\subsection{Super-linear Rate for Bandits}
\label{app:tab-mab-spma-super}
In order to achieve the desired fast rate of convergence, we modify the update in~\cref{eq:tab-mab-spma} to use a set of ${K \choose 2}$ constant gap-dependent step-sizes $\{\eta_{a,a'}\}_{a,a' \in [K]}$. The new update can be written as:
\begin{align}
\ppitt(a) &= \ppit(a) \; [1 + \sum_{a' \neq a} \ppit(a') \, \eta_{a,a'} \, \Delta(a, a') ]
\label{smdpo-bandit-update-gap}    
\end{align}
The following theorem shows that the above update results in super-linear convergence.
\begin{restatable}{theorem}{deltadepbandit}
\label{delta-dependent-eta-bandit}
    Using the $\Alg$ update in~\cref{smdpo-bandit-update-gap} with (i) $\eta_{a, a'} = \frac{1}{|\Delta(a, a')|}$ and a (ii) uniform initialization similar to~\cref{thm:tab-spma-bandit-linear} results in valid probability distributions and converges as:
    \begin{align*}
        r(a^*) - \inner{\pi_T}{r}  \leq \left[ \left(1 - \frac{1}{K} \right)\right]^{2^T}
    \end{align*}
    where $T$ is the number of iterations, $a^*$ is the optimal arm and $\Delta(a, a') := r(a) - r(a')$ represents the reward gap between arms $a$ and $a'$.  
\end{restatable}

\begin{proof}
We define $\Delta(a, a'):= r(a) - r(a')$,  
\begin{align*}
    \apit & = r(a) - \langle \ppit, r \rangle \\
    & =\sum_{a'} \ppit(a')[r(a) - r(a')]\\
    & = \sum_{a'} \ppit(a')\Delta(a,a')
\end{align*}


Choosing different step sizes for every pair of arms, depending on their corresponding gap, $\eta_{a, a'} = \frac{1}{|\Delta(a, a')|}$ we get the following update for $\ppitt(a)$:
\begin{align*}
    \ppitt(a) &= \ppit(a)\left[ 1 + \sum_{a' \ne a} \eta_{a, a'} \ppit(a') \Delta(a, a')\right]\\
              &= \ppit(a)\left[ 1 + \sum_{a' \ne a} \ppit(a') \sign{(\Delta(a, a'))}\right] \label{smdpo-bandit-update-expanded}\tag{i}\\
\end{align*}
Now we check if $\ppitt$ is a probability distribution with this choice of $\eta$. Note that $\Delta(a,a')= - \Delta(a',a)$. 
\begin{align*}
    \sum_a \ppitt(a)& = \sum_a \ppit (a) + \sum_a  \ppit(a) \sum_{a' \ne a} \ppit(a') \sign{(\Delta(a, a'))} \\
    & = 1 + \sum_{(a,a'), a\neq a'} \ppit(a) \ppit(a') ( \sign{(\Delta(a, a'))} + \sign{(\Delta(a', a))}) \\
    & = 1 + \sum_{(a,a'), a\neq a'} \ppit(a) \ppit(a') ( \sign{(\Delta(a, a'))} - \sign{(\Delta(a, a'))}) \tag{ since $\Delta(a,a')= - \Delta(a',a)$} \\
    & = 1 
\end{align*}
Furthermore, it is clear that  $\ppit(a) \in [0,1]$. Based on this we just need to show that the probability of the optimal arm $a^*$ converges to $1$.   

Computing the probability of pulling the optimal arm using update (i):
\begin{align*}
    \ppitt(a^*) &= \ppit(a^*)\left[ 1 + \sum_{a' \ne a^*} \ppit(a') \sign{(\Delta(a^*, a'))}\right]\\
                &= \ppit(a^*)\left[ 1 + \sum_{a' \ne a^*} \ppit(a') \right] \tag{$\Delta(a^*, a') > 0 \,\,\, \forall a'$}\\ 
                &= \ppit(a^*)\left[ 2 - \ppit(a^*) \right] \tag{ii}\\
\end{align*}
We use induction to show $\ppit(a^*) = 1-\left[(1 - \frac{1}{K})\right]^{2^t}$ solves the recurrence relation (ii). We consider the uniform distribution over the arms at the initialization i.e. $\pi_0(a) = \frac{1}{K}, \,\,\, \forall a \in \cA$. For the base case, we show the suggested solution satisfies recursion (ii): 
\begin{align*}
    \pi_1(a^*) &= \frac{1}{K} \, \left(2 - \frac{1}{K}\right) \tag{using the recursion in (ii)}\\
                   &= \left(1 - 1 + \frac{1}{K}\right) \left(1 + 1 - \frac{1}{K}\right) \\
                   &= 1 - \left[\left(1 - \frac{1}{K}\right)\right]^2
\end{align*}
Assuming the suggested solution is true for $t$, we show it is also true for $t+1$:
\begin{align*}
    \ppitt(a^*) &= \left[1 - \left(1 - \frac{1}{K}\right)^{2^t}\right] \, \left[2 - 1 + \left(1 - \frac{1}{K}\right)^{2^t}\right]\\
                &= 1 - \left[\left(1 - \frac{1}{K}\right)^{2^{t+1}}\right]
\end{align*}
Let $\delta_T:= r(a^*) - \langle \pi_T, r \rangle$ represent the sub-optimality gap.
\begin{align*}
    \delta_T &= \sum_{a'} \pi_T(a') \left[r(a^*) - r(a')\right]\\
             &= \sum_{a' \ne a^*} \pi_T(a)\Delta(a^*, a)\\
             &\leq \max_{a'} \Delta(a^*, a') \sum_{a' \ne a^*} \pi_T(a)\\
             &\le 1-\pi_{T}(a^*) \tag{using the fact $0\le r\le1$}\\
             &= \left[\left(1 - \frac{1}{K}\right)\right]^{2^T}\tag{using the formula for $\pi_{T}(a^*)$} 
\end{align*} 
\end{proof}


\section{MDP Proofs}
\label{app:mdp-proof}
\subsection{Tabular Setting}
\label{app:mdp-tab}
\begin{lemma}
    \label{lem:avg-adv}
    For any policy $\ppit$ we have 
    \begin{align*}
         \sum_{a} \ppit(a|s) [\apit(s, a)]^2 \ge \,\, C_t \max_{a} \apit(s, a)
    \end{align*}
    where $C_t := \min_s \{ \ppit(\tilde{\cA}_t(s) | s)  \, \Delta^t(s) \}$, $\tilde{\cA}_t(s) := \argmax_{a \in \cA} \qpit(s, a)$, $\ppit(\tilde{\cA}_t(s) | s)= \sum_{a \in \tilde{\cA}_t(s)} \ppit(a_t(s) | s)$ and $\Delta^t(s) := \max_{a\in \cA} \qpit(s, a) - \max_{a\notin \tilde{\cA}} \qpit(s, a)$.
\end{lemma}
\begin{proof}
    Recall $\tilde{\cA}_t(s) := \argmax_{a\in \cA} \apit(s, a)$ i.e. $\tilde{\cA}_t(s)$ is a set containing actions with maximum advantage for state $s$. Let's define $ \ppit(\tilde{\cA}_t(s)|s) = \sum_{a \in\tilde{\cA}_t(s) }\ppit(\tilde{a}_t(s)|s)$. We can split the sum on the LHS of the above over $\tilde{\cA}_t(s)$:
    \begin{equation}
    \begin{aligned}
         \sum_{a} \ppit(a|s) [\apit(s, a)]^2 &=  \,\, \sum_{a \in \tilde{\cA}_t(s) }\ppit(\tilde{a}_t(s)|s) [\max_{a} \apit(s, a)][\max_{a} \apit(s, a)]\\
        &+  \sum_{a \notin \tilde{\cA}_t(s)} \ppit(a|s) [\apit(s, a)]^2 \\
        & = \ppit(\tilde{\cA}_t(s)|s) [\max_{a} \apit(s, a)][\max_{a} \apit(s, a)]\\
        &+  \sum_{a \notin \tilde{\cA}_t(s)} \ppit(a|s) [\apit(s, a)]^2 \label{split-over-tilde-a}
    \end{aligned}
    \end{equation}
    Let $\tilde{\ppit}$ be the following distribution over the actions.
    \begin{align*}
    \tilde{\ppit}(a | s)= 
        \begin{cases}
        0     & \text{if } a \in \tilde{\cA}_t(s)\\
        \frac{\ppit(a|s)}{1-\ppit(\tilde{\cA}_t(s)| s)}   & \text{otherwise}
        \end{cases}
    \end{align*}
    Re-writing $\sum_a \ppit(a|s) \apit(s, a) = 0$ using the above distribution we obtain:
    \begin{align*}
        (1-\ppit(\tilde{\cA}_t(s)|s))\,\, \E_{a\sim\tilde{\ppit}} [\apit(s, a)] + \ppit(\tilde{\cA}_t(s)|s)[\max_{a} \apit(s, a)] = 0
    \end{align*}
    \begin{align}
        (1-\ppit(\tilde{\cA}_t(s)|s))\,\, \E_{a\sim\tilde{\ppit}} [\apit(s, a)] = -\ppit(\tilde{\cA}_t(s)|s)[\max_{a} \apit(s, a)] \label{pi-tilde-relation}
    \end{align}
    Expanding the second term in Eq.~\ref{split-over-tilde-a} using $\tilde{\ppit}$ we obtain:
    \begin{align*}
         \sum_{a \notin \tilde{\cA}_t(s)} \ppit(a|s) [\apit(s, a)]^2 &=  \,\, (1-\ppit(\tilde{\cA}_t(s)|s))\,\, \E_{a\sim\tilde{\ppit}} [\apit(s, a)]^2\\
        &\ge  \,\, (1-\ppit(\tilde{\cA}_t(s)|s))\,\, \left(\E_{a\sim\tilde{\ppit}} [\apit(s, a)]\right)^2 \tag{using $\E[x^2] \ge (\E[x])^2$}\\
        &=  \,\, (1-\ppit(\tilde{\cA}_t(s)|s))\,\, \left(\E_{a\sim\tilde{\ppit}} [\apit(s, a)]\right) \left(\E_{a\sim\tilde{\ppit}} [\apit(s, a)]\right)\\
        &= -\,\, \ppit(\tilde{\cA}_t(s)|s)[\max_a \apit(s, a)]\left(\E_{a\sim\tilde{\ppit}} [\apit(s, a)]\right) \tag{using Eq.~\ref{pi-tilde-relation}}\\
    \end{align*}
    Plugging in the result above into Eq.~\ref{split-over-tilde-a} we obtain:
    \begin{align*}
         \sum_{a} \ppit(a|s) [\apit(s, a)]^2 &\geq  \,\, \ppit(\tilde{\cA}_t(s)|s) [\max_{a} \apit(s, a)][\max_{a} \apit(s, a)]\\
        &-  \,\, \ppit(\tilde{\cA}_t(s)|s)[\max_a \apit(s, a)]\left(\E_{a\sim\tilde{\ppit}} [\apit(s, a)]\right)\\
        &\geq  \,\, \ppit(\tilde{\cA}_t(s)|s) [\max_a \apit(s, a)] \left[ \max_a \apit(s, a) - \E_{a\sim\tilde{\ppit}} [\apit(s, a)]\right]\\  
        & \geq  \,\, \ppit(\tilde{\cA}_t(s)|s) [\max_a \apit(s, a)] \left[ \underbrace{\max_a \apit(s, a) - \max_{a\notin \tilde{\cA}} \apit(s, a)}_{:=\Delta^t(s)}\right]\\
        & =  \,\, \ppit(\tilde{\cA}_t(s)|s) [\max_a \apit(s, a)] \Delta^t(s)\\
        &\ge  \,\, C_t \max_a \apit(s, a)
    \end{align*}
\end{proof}

\begin{lemma}
\label{lem:lin-improve}
     Using the update $\ppitt(a | s) = \ppit(a | s) \, (1 + \eta \apit(s, a)) $ with a step-size $\eta < \min\left\{1 - \gamma, \frac{1}{C_t (1-\gamma)} \right\}$, at any iteration $t$ and state $s \in S$, we have
      \begin{align*}
          V^{*}(s) - V^{\pi_{t+1}}(s) \leq \left[1 - \eta\,\,C_t(1-\gamma)\right]\left[V^{*}(s) - V^{\pi_{t}}(s)\right]
      \end{align*}  
      where $C_t := \min_s \{ \ppit(\tilde{\cA}_t(s) | s)  \, \Delta^t(s) \}$, $\tilde{\cA}_t(s) := \argmax_a \qpit(s, a)$, $\ppit(\tilde{\cA}_t(s) | s)= \sum_{a \in \tilde{\cA}_t(s)} \ppit(a_t(s) | s)$, $\Delta^t(s) := \max_a \qpit(s, a) - \max_{a\notin \tilde{\cA}} \qpit(s, a)$, and $V^*(s)$ is the value function corresponding to the optimal policy $\pi^*$ at $s \in \cS$.
\end{lemma}
\begin{proof}
    First, we use the value difference Lemma to show the $\Alg$ update in~\cref{eq:tab-mdp-spma} leads to a monotonic improvement in the value function.
    \begin{align}
        V^{\pi_{t+1}}(s) - V^{\pi_{t}}(s) = \frac{1}{1-\gamma} \E_{s \sim \dpitt} \left[ \sum_{a} \ppitt(a|s) \apit(s, a) \right] \label{value-diff-lemma}
    \end{align}
    Plugging update~\cref{eq:tab-mdp-spma} into the term within the brackets, we obtain the following:
    \begin{align*}
        \sum_{a} \ppitt(a|s) \apit(s, a) &= \sum_{a} \ppit(a|s) \apit(s, a)[1 + \eta \apit(s, a)]\\
        &= \sum_{a} \ppit(a|s) \apit(s, a) + \eta \sum_{a} \ppit(a|s) [\apit(s, a)]^2\\
        &= \eta \sum_{a} \ppit(a|s) [\apit(s, a)]^2\\
        & > 0
    \end{align*}
    Hence, $V^{\pi_{t+1}}(s) \geq V^{\pi_t}(s)$. Using~\cref{lem:avg-adv}, we have: 
    \begin{align}
        \eta \sum_{a} \ppit(a|s) [\apit(s, a)]^2 \ge \eta \,\, C_t \max_{a} \apit(s, a)
    \end{align}
    
    Combining the above with the result from the value difference Lemma we have:
    \begin{equation}
    \begin{aligned}
        \sum_{a} \ppitt(a|s) \apit(s, a) &= \eta \sum_{a} \ppit(a|s) [\apit(s, a)]^2\\
        &\ge \eta \,\, C_t \max_a \apit(s, a) \label{ct-lower-bound}
    \end{aligned}
    \end{equation}
    We now show a linear convergence when using the update in~\cref{eq:tab-mdp-spma}. Let $T$ be the Bellman optimality operator defined as:
    \begin{align*}
        (Tv)(s) = \max_a \{r(s, a) + \gamma \sum_{s'} \Pr[s'|s, a] v(s')\}
    \end{align*}
    Applying the operator at iteration $t$ we obtain:
    \begin{equation}
    \begin{aligned}
        TV^{\pi_{t}}(s) - V^{\pi_{t}}(s) = \max_a \qpit(s, a) - V^{\pi_{t}}(s) = \max_a \apit(s, a) \label{bellman-diff}
    \end{aligned}
    \end{equation}
    Let $\tpi$ be an operator w.r.t $\pi$ defined as:
    \begin{align*}
        \tpi(v) = \sum_a \pi(a|s)r(s, a) + \gamma \sum_a \pi(a|s)\sum_{s'}\Pr[s'|s, a] v(s')
    \end{align*}
    Applying $\tpi$ to $V^{\pi'}(s)$ results in:
    \begin{align*}
        \tpi V^{\pi'}(s) &=  \sum_a \pi(a|s)r(s, a) + \gamma \sum_a \pi(a|s)\sum_{s'}\Pr[s'|s, a] V^{\pi'}(s)\\   
                       &= \sum_a \pi(a|s) Q^{\pi'}(s, a)
    \end{align*}
    Using the above we obtain:
    \begin{align*}
        \tpitt V^{\pi_{t}}(s) - V^{\pi_{t}}(s) &= \sum_a \ppitt(a|s) \apit(s, a)\\
        &\ge \eta \,\, C_t \max_a \apit(s, a) \tag{using Ineq.~\ref{ct-lower-bound}}\\
        &= \eta \,\, C_t \left[TV^{\pi_{t}}(s) - V^{\pi_{t}}(s)\right] \tag{using Eq.~\ref{bellman-diff}}
    \end{align*}
    Assuming $\pi^*$ is the optimal policy we have:
    \begin{align*}
        V^{*}(s) - V^{\pi_{t+1}}(s) &= V^{*}(s) - \tpitt V^{\pi_{t+1}}(s) \tag{since $\tpi V^{\pi}(s) = V^{\pi}(s)$}\\
                            &\le V^{*}(s) - \tpitt V^{\pi_{t}}(s) \tag{since $V^{\pi_{t+1}}(s) \ge V^{\pi_{t}}(s) \,\,\, \forall s$}\\
                            &= V^{*}(s) - V^{\pi_{t}}(s) - \left[\tpitt V^{\pi_{t}}(s) - V^{\pi_{t}}(s) \right] \tag{add and subtract $V^{\pi_{t}}(s)$}\\
                            &\le V^{*}(s) - V^{\pi_{t}}(s) -\eta \,\, C_t \left[TV^{\pi_{t}}(s) - V^{\pi_{t}}(s)\right]\\
                            &= \eta\,\,C_t[V^{*}(s) - V^{\pi_{t}}(s)] + (1 - \eta\,\,C_t)[V^{*}(s) - V^{\pi_{t}}(s)] - \eta\,\,C_t \left[TV^{\pi_{t}}(s) - V^{\pi_{t}}(s)\right]\\
                            &= \eta\,\,C_t \left[TV^{*}(s) - V^{\pi_{t}}(s) - T V^{\pi_{t}}(s) + V^{\pi_{t}}(s) \right] + (1-\eta\,\,C_t)\left[V^{*}(s) - V^{\pi_{t}}(s)\right]\\
                            &=\eta\,\,C_t \left[TV^{*}(s) - T V^{\pi_{t}}(s)\right] + (1-\eta\,\,C_t)\left[V^{*}(s) - V^{\pi_{t}}(s)\right]\\
                            &\le \gamma \,\,\eta\,\,C_t \left[V^{*}(s) - V^{\pi_{t}}(s)\right] + (1-\eta\,\,C_t)\left[V^{*}(s) - V^{\pi_{t}}(s)\right] \tag{$T$ is a $\gamma$ contraction map}\\
                            &= \left[1 - \eta\,\,C_t(1-\gamma)\right]\left[V^{*}(s) - V^{\pi_{t}}(s)\right]
    \end{align*}
\end{proof}

\smdpoconstmdp*
\begin{proof}
    Using~\cref{lem:lin-improve} we have 
    \begin{align*}
            V^{*}(s) - V^{\pi_{t+1}}(s) \leq \left[1 - \eta\,\,C_t(1-\gamma)\right]\left[V^{*}(s) - V^{\pi_{t}}(s)\right]
    \end{align*}    
    If $\eta < \frac{1}{C_t(1-\gamma)}$, both sides of the inequality above are positive leading to $|V^{*}(s) - V^{\pi_{t+1}}(s)| \le (1 - \eta\,\,C_t(1-\gamma)) |V^{*}(s) - V^{\pi_{t}}(s)|$. This is true for all $s \in \cS$, hence we have:
    \begin{align*}
        \norminf{V^{\pi^*} - V^{\pi_{t+1}}} &\le (1 - \eta\,\,C_t(1-\gamma)) \norminf{V^{\pi^*} - V^{\pi_t}}\\
                                    &= \alpha_t \norminf{V^{\pi^*} - V^{\pi_t}}
    \end{align*}
    Recursing from $t=0$ to $t=T-1$ we obtain a linear convergence:
    \begin{align*}
        \norminf{V^{\pi^*} - V^{\pi_T}} &\le \left(\prod_{t=0}^{T-1} \alpha_t\right) \norminf{V^{\pi^*} - V^{\pi_0}}
    \end{align*}
\end{proof}

\subsection{Function Approximation With Exact Advantage}
\label{mdp-func-exact-adv}
Recall the definitions of $\tilde{\ell_t}$ and $\ell_t$
\begin{align*}
    & \tilde{\ell}_t(\theta) = \sum_{s} d^{\ppit}(s) \, \text{KL}(\ppith(\cdot|s) \, || \, \pi_{\theta}(\cdot|s)) \label{eq:spma-surrogate} \\
& \ell_t(\theta) = \sum_{s \sim \tau} \text{KL} \left(h(f_{\tht}(s,\cdot))  (1 + \eta  A^{\ppit}(s,\cdot)) \, || \, h(f_{\theta}(s,\cdot))  \right)\\
\end{align*}


\smdpofuncapp*

\begin{proof}
We assumed that $z_{\theta}(s,a) := f_{\theta}(s,a)\, \, \forall (s,a)$ and $z_t(s,a)=f_{\theta_{t+1}}(s,a)$ where $f_{\theta}: \R^{SA} \rightarrow \R$ is a complex, non-linear model. We remind the following updates: 

\begin{align*}
    &z_{t+1/2} = \argmax_{\bar{z} \in \mathbb R^{|S||A|}}{\left\{ \inner{\nabla_z J(z_t)}{z} -1/\eta D_{\Phi}(z,z_t)\right\}}\\
    & \nabla \Phi(z_{t+1/2}) = \nabla \Phi(z_t) + \eta \nabla_z J(z_t) \tag{Mirror Ascent update without projection}\\
    &\bppitt = h(z_{t+1/2})\tag{$h$ is softmax}\\
    &\bppitt(a|s) = \ppit(a|s)(1+\eta A^\ppit(s,a)) \\
    & \theta_{t+1} = \text{(S)GD}(\ell_t(\theta))  \tag{using (Stochastic)Gradient Descent for $m$ iteration to minimize $\ell_t$}\\
    &\ppitt = h(z_{t+1})\\
\end{align*}
$z_{t+1/2}$ is an unprojected update for the tabular setting and therefore using~\cref{lem:lin-improve} we have: 
\begin{align*}
    V^{*}(s) - V^{\bppitt}(s) &\leq   \left[1 - \eta\,\,C_t(1-\gamma)\right]\left[V^{*}(s) - V^{\pi_{t}}(s)\right]
\end{align*}
By adding and removing $V^{\pi_{t+1}}(s)$ to both sides and rearranging we have 
\begin{align*}
    V^{*}(s) - V^{\pi_{t+1}}(s) & \leq  \left[1 - \eta\,\,C_t(1-\gamma)\right]\left[V^{*}(s) - V^{\pi_{t}}(s)\right] + V^{\bppitt}(s) - V^{\pi_{t+1}}(s).\\
\end{align*}
Taking the expectation w.r.t. $\rho$, we obtain: 
\begin{align*}
    J({\pi^*}) - J(\pi_{t+1}) & \leq  \left[1 - \eta\,\,C_t(1-\gamma)\right]\left[J(\pi^*) - J(\pi_{t})\right] + \underbrace{J(\pi_{t+1/2}) - J(\pi_{t+1})}_{:=E_1}\\
\end{align*}
The term $E_1$ can be bounded as follows:  
\begin{align*}
    E_1 &=\sum_s \dbpitt(s) \inner{\qpitt(s,.)}{\bppitt(.|s)-\ppitt(.|s)}\\
    & \leq \sum_s \dbpitt(s) \norm{\qpitt(s,.)}_{\infty}\norm{\bppitt(.|s)-\ppitt(.|s)}_{1}\tag{Holder inequality}\\
    & \leq \frac{1}{1-\gamma}\sum_s \dbpitt(s)\norm{\bppitt(.|s)-\ppitt(.|s)}_{1} \tag{since $\norm{\qpitt(s,.)}_{\infty} \leq \frac{1}{1-\gamma}$}\\
    & =  \frac{1}{1-\gamma}\sum_s \frac{\dbpitt(s)}{\rho(s)}\rho(s) \norm{\bppitt(.|s)-\ppitt(.|s)}_{1}\\ 
    & \leq \frac{1}{1-\gamma} \norm{\frac{\dbpitt}{\rho}}_{\infty} \sum_s \rho(s) \norm{\bppitt(.|s)-\ppitt(.|s)}_{1}\\ 
   & \leq \frac{1}{(1-\gamma)\rho_{min}} \sum_s \rho(s) \norm{\bppitt(.|s)-\ppitt(.|s)}_{1}\tag{since $\dbpitt(s) \leq 1$ and using~\cref{as:exploration}}\\
   & \leq \frac{1}{(1-\gamma)^2\rho_{min}} \sum_s \dpit(s) \norm{\bppitt(.|s)-\ppitt(.|s)}_{1} \tag{sicne $\dpit \geq (1-\gamma) \rho$}\\
   & \leq \frac{\sqrt{2}}{(1-\gamma)^2\rho_{min}} \sum_s \dpit(s) \sqrt{ \text{KL}(\bppitt(.|s)\, || \, \ppitt(.|s))} \tag{using strong convexity of $ \text{KL}$ divergence or Pinsker's inequality}\\
   & \leq \frac{\sqrt{2}}{(1-\gamma)^2\rho_{min}} \sqrt{\underbrace{\sum_s \dpit(s)  \text{KL}(\bppitt(.|s)\, || \, \ppitt(.|s))}_{:=E_2}} \tag{due to concavity of $\sqrt{ }$ and Jensen's inequality}
\end{align*}
where $E_2$ can be bounded as follows. 
\begin{align*}
    E_2 &= \tilde{\ell}_t(\theta_{t+1}) \\
        &= \tilde{\ell}_t(\theta_{t+1}) - \min_{\theta}\tilde{\ell}_t(\theta_{t+1})+\min_{\theta}\tilde{\ell}_t(\theta_{t+1})\\
        & \leq \epsilon_{stat} +\min_{\theta}\tilde{\ell}_t(\theta_{t+1})\tag{using~\cref{as:samplingerror}}\\
        & \leq \epsilon_{stat} + \epsilon_{bias}\tag{using~\cref{as:bias}}
\end{align*}

Putting everything together we have: 
\begin{align*}
    E_1 \leq \underbrace{\frac{\sqrt{2}}{(1-\gamma)^2\rho_{min}} \sqrt{\epsilon_{\text{stat}}+ \epsilon_{\text{bias}}}}_{:=\beta}
\end{align*}
Therefore we have 
\begin{align*}
    J(\pi^*) - J(\pi_{t+1}) & \leq  \underbrace{\left[1 - \eta\,\,C_t(1-\gamma)\right]}_{\alpha_t}\left[J(\pi^*) - J(\pi_{t})\right] + \beta
\end{align*}
Unrolling the above recursion for $T$ iterations, 
\begin{align*}
    J(\pi^*) - J(\pi_{T}) \leq \left(\prod_{t=0}^{T-1} \alpha_t\right) (J(\pi^*) - J(\pi_0)) + \beta \sum_{t=0}^{T-1}\prod_{i=t+1}^{T-1}\alpha_t 
\end{align*}
\end{proof}

\subsection{Function Approximation With Inexact Advantage}
\label{mdp-func-inexact-adv}
In practice computing the $A^{\ppit}$ at every iteration is costly. In this section, we assume that we access an oracle such that at each iteration $t$, it gives us $\hat{A}^{\ppit}$ an approximation of $A^{\ppit}$.  

\begin{assumption}
\label{as:vaildapprox}
\textit{Valid Approximation}. For all iterations $t$ and $(s,a) \in \cS \times \cA$, $|\hat{A}^{\ppit}(s,a)| \leq \frac{1}{1-\gamma}$.  
\end{assumption}

\begin{assumption}
\label{as:approxerror}
\textit{Approximation Error}. For all iterations $t$ and $s \in \cS$, $\norm{A^{\ppit}(s,.)-\hat{A}^{\ppit}(s,.)}_{\infty}\leq \epsilon_{approx}$.  
\end{assumption}

Using this inexact advantage function, we define the following update and functions 
\begin{align*}
    &\bppitt(a|s) = \ppit(a|s)(1+\eta \hat{A}^\ppit(s,a))\tag{replacing $A^{\ppit}$ with $\hat{A}^{\ppit}$ }\\        
    & \tilde{\ell}_t(\theta) = \sum_{s} d^{\ppit}(s) \, \text{KL}(\ppith(\cdot|s) \, || \, \pi_{\theta}(\cdot|s))  \\
    & \ell_t(\theta) = \sum_{s \sim \tau} \text{KL} \left(h(f_{\tht}(s,\cdot))  (1 + \eta  \hat{A}^{\ppit}(s,\cdot))\, || \, h(f_{\theta}(s,\cdot))  \right)
\end{align*}

Since we use the inexact advantage, we cannot reuse the result of~\cref{lem:lin-improve}. So we provide a variant of that lemma with an inexact advantage. 

\begin{lemma}
    \label{lem:inex-lin-improv}
    Using the update $\ppitt(a | s) = \ppit(a | s) \, (1 + \eta \hapit(s, a)) $ with (i) a step-size $\eta < \min\left\{1 - \gamma, \frac{1}{C_t (1-\gamma)} \right\}$ and (ii) $\hapit$ satisfying assumptions~\ref{as:vaildapprox} and~\ref{as:approxerror}, at any iteration $t$ and $s \in S$ we have
      \begin{align*}
          V^{*}(s) - V^{\pi}(s) \leq \left[1 - \eta\,\,C_t(1-\gamma)\right]\left[V^{*}(s) - V^{\pi_{t}}(s)\right]+ \frac{\epsilon_{approx}}{1-\gamma}
      \end{align*}  
       where $C_t := \min_s \{ \ppit(\tilde{\cA}_t(s) | s)  \, \Delta^t(s) \}$, $\tilde{\cA}_t(s) := \argmax_a \qpit(s, a)$, $\ppit(\tilde{\cA}_t(s) | s)= \sum_{a \in \tilde{\cA}_t(s)} \ppit(a_t(s) | s)$ and $\Delta^t(s) := \max_a \qpit(s, a) - \max_{a\notin \tilde{\cA}} \qpit(s, a)$, and $V^*(s)$ is the value function corresponding to the optimal policy $\pi^*$ at $s \in \cS$.
\end{lemma}

\begin{proof}
First, we use the value difference Lemma for a state $s \in \cS $
    \begin{align*}
        V^{\pi}(s) - \jspit &= \frac{1}{1-\gamma} \E_{s' \sim \dpi} \left[ \sum_{a} \pi(a|s') \apit(s', a) \right] \\
        & = \frac{1}{1-\gamma} \E_{s' \sim \dpi} \left[ \sum_{a} \ppit(a | s') \, (1 + \eta \hapit(s', a))  \apit(s', a) \right] \\
        &= \frac{1}{1-\gamma} \E_{s' \sim \dpi} \left[ \sum_{a} \eta \ppit(a | s') \,  \hapit(s', a)  \apit(s', a) \right] \\
        & = \frac{1}{1-\gamma} \E_{s' \sim \dpi} \left[ \sum_{a} \eta \ppit(a | s') \,   (\hapit(s', a)-\apit(s',a) +\apit(s',a))  \apit(s', a) \right] \\
        & = \frac{1}{1-\gamma} \E_{s' \sim \dpi} \underbrace{\left[ \sum_{a} \eta \ppit(a | s') \, (\apit(s',a))^2  \right]}_{:=T_1} \\
        &+ \frac{1}{1-\gamma}\E_{s' \sim \dpi}  \underbrace{\left[ \sum_{a} \eta \ppit(a | s') \,  (\hapit(s', a)-\apit(s',a) )  \apit(s', a) \right]}_{:=T_2} 
    \end{align*} 
$T_1$ can be bounded using~\cref{lem:avg-adv}, 
\begin{align*}
    T_1 \geq \,\,\eta  C_t \max_{a} \apit(s, a) 
\end{align*}
To bound $T_2$, we use $\cref{as:approxerror}$, 
\begin{align*}
    T_2 &\geq - \eta  \left[ \sum_{a} \ppit(a | s') \,  |(\hapit(s', a)-\apit(s',a) )|  |\apit(s', a)| \right]\\
    & \geq - \eta  \left[ \sum_{a} \ppit(a | s') \,  |(\hapit(s', a)-\apit(s',a) )| \frac{1}{1-\gamma} \right]\tag{since $A^{\pi} \leq 1/(1-\gamma)$}\\
    & \geq - \eta  \left[ \sum_{a} \ppit(a | s') \,   \frac{\epsilon_{approx}}{1-\gamma} \right]\tag{using~\cref{as:approxerror}}\\
    & = - \frac{\eta\epsilon_{approx}}{1-\gamma}
\end{align*}
Using the lower-bound for $T_1$ and $T_2$ we have 
\begin{align}
    \sum_{a} \pi(a|s) \apit(s,a) \geq \eta  C_t \max_{a} \apit(s, a) - \frac{\eta\epsilon_{approx}}{1-\gamma}\label{eq:avgtt-adv}
\end{align}

Putting everything together we have, 
\begin{align}
    V^{\pi}(s) & \geq  \jspit - \frac{\eta\epsilon_{approx}}{(1-\gamma)^2}\\
               &  \geq \jspit - \frac{\epsilon_{approx}}{(1-\gamma)} \tag{since $\eta \leq 1- \gamma $}  \\
    \implies  \jspit-  V^{\pi}(s) &\leq   \frac{\epsilon_{approx}}{(1-\gamma)} \label{ineq:dif-lower}\\        
\end{align}
    Let $T$ be the Bellman optimality operator defined as:
    \begin{align*}
        (Tv)(s) = \max_a \{r(s, a) + \gamma \sum_{s'} \Pr[s'|s, a] v(s')\}
    \end{align*}
    Applying the operator at iteration $t$ we obtain:
    \begin{equation}
    \begin{aligned}
        TV^{\pi_{t}}(s) - V^{\pi_{t}}(s) = \max_a \qpit(s, a) - V^{\pi_{t}}(s) = \max_a \apit(s, a) \label{eq:bellman-diff2}
    \end{aligned}
    \end{equation}
    Let $\tpi$ be an operator w.r.t $\pi$ defined as:
    \begin{align*}
        \tpi(v) = \sum_a \pi(a|s)r(s, a) + \gamma \sum_a \pi(a|s)\sum_{s'}\Pr[s'|s, a] v(s')
    \end{align*}
    Applying $\tpi$ to $V^{\pi'}(s)$ results in:
    \begin{align*}
        \tpi V^{\pi'}(s) &=  \sum_a \pi(a|s)r(s, a) + \gamma \sum_a \pi(a|s)\sum_{s'}\Pr[s'|s, a] V^{\pi'}(s)\\   
                       &= \sum_a \pi(a|s) Q^{\pi'}(s, a)
    \end{align*}
    Using the above we obtain:
    \begin{align*}
        \tpi V^{\pi_{t}}(s) - V^{\pi_{t}}(s) &= \sum_a \ppi(a|s) \apit(s, a)\\
        &\ge \eta \,\, C_t \max_a \apit(s, a)- \frac{\eta\epsilon_{approx}}{1-\gamma} \tag{using ~\cref{eq:avgtt-adv}}\\
        &= \eta \,\, C_t \left[TV^{\pi_{t}}(s) - V^{\pi_{t}}(s)\right]- \frac{\eta\epsilon_{approx}}{1-\gamma} \tag{using Eq.~\ref{eq:bellman-diff2}}\\
        & \geq \eta \,\, C_t \left[TV^{\pi_{t}}(s) - V^{\pi_{t}}(s)\right]- \epsilon_{approx}\tag{since $\eta \leq 1-\gamma$}
    \end{align*}
Using~\cref{ineq:dif-lower} we have 
\begin{align}
       \tpi V^{\pi_{t}}(s) - \tpi V^{\pi}(s)  
        & = \gamma \sum_a \pi(a|s)\sum_{s'}\Pr[s'|s, a] (V^{\ppit}(s')-V^{\pi}(s'))\\
        & \leq \gamma \sum_a \pi(a|s)\sum_{s'}\Pr[s'|s, a] \frac{\epsilon_{approx}}{1-\gamma} \tag{using ~\cref{ineq:dif-lower}} \\ 
        & =  \frac{\gamma \epsilon_{approx}}{1-\gamma} \label{eq:bel-dif-up}
\end{align}

Assuming $\pi^*$ is the optimal policy we have:

 \begin{align*}
        V^{*}(s) - V^{\pi}(s) &= V^{*}(s) - \tpi V^{\pi}(s) \tag{since $\tpi V^{\pi}(s) = V^{\pi}(s)$}\\
                            &= V^{*}(s) - \tpi V^{\pi_{t}}(s)+ \tpi V^{\pi_{t}}(s)- \tpi V^{\pi}(s) \\
                            &\le V^{*}(s) - \tpi V^{\pi_{t}}(s)+ \frac{\gamma \epsilon_{approx}}{1-\gamma}\tag{using~\cref{eq:bel-dif-up}}\\
                            &= V^{*}(s) - V^{\pi_{t}}(s) - \left[\tpi V^{\pi_{t}}(s) - V^{\pi_{t}}(s) \right] + \frac{\gamma \epsilon_{approx}}{1-\gamma}\tag{add and subtract $V^{\pi_{t}}(s)$}\\
                            &\le V^{*}(s) - V^{\pi_{t}}(s) -\eta \,\, C_t \left[TV^{\pi_{t}}(s) - V^{\pi_{t}}(s)\right]+\epsilon_{approx}+ \frac{\gamma \epsilon_{approx}}{1-\gamma}\\
                            &= \eta\,\,C_t[V^{*}(s) - V^{\pi_{t}}(s)] + (1 - \eta\,\,C_t)[V^{*}(s) - V^{\pi_{t}}(s)] - \eta\,\,C_t \left[TV^{\pi_{t}}(s) - V^{\pi_{t}}(s)\right]+ \frac{\epsilon_{approx}}{1-\gamma}\\
                            &= \eta\,\,C_t \left[TV^{*}(s) - V^{\pi_{t}}(s) - T V^{\pi_{t}}(s) + V^{\pi_{t}}(s) \right] + (1-\eta\,\,C_t)\left[V^{*}(s) - V^{\pi_{t}}(s)\right]+ \frac{ \epsilon_{approx}}{1-\gamma}\\
                            &=\eta\,\,C_t \left[TV^{*}(s) - T V^{\pi_{t}}(s)\right] + (1-\eta\,\,C_t)\left[V^{*}(s) - V^{\pi_{t}}(s)\right]+ \frac{ \epsilon_{approx}}{1-\gamma}\\
                            &\le \gamma \,\,\eta\,\,C_t \left[V^{*}(s) - V^{\pi_{t}}(s)\right] + (1-\eta\,\,C_t)\left[V^{*}(s) - V^{\pi_{t}}(s)\right]+ \frac{ \epsilon_{approx}}{1-\gamma} \tag{$T$ is a $\gamma$ contraction map}\\
                            &= \left[1 - \eta\,\,C_t(1-\gamma)\right]\left[V^{*}(s) - V^{\pi_{t}}(s)\right]+ \frac{ \epsilon_{approx}}{1-\gamma}
    \end{align*}

\end{proof}

\begin{restatable}{theorem}{smdpofuncappinexadv}
Under~\cref{as:samplingerror}-\ref{as:approxerror},~\cref{alg:spma} with $\eta < \min\left\{1 - \gamma, \frac{1}{C_t (1-\gamma)} \right\}$ converges as,
\begin{align*}
& J(\pi^*) - J(\pi_T) \\ & \leq \left(\prod_{t=0}^{T-1} \alpha_t\right) (J(\pi^*) - J(\pi_0)) + \beta \sum_{t=0}^{T-1}\prod_{i=t+1}^{T-1}\alpha_t \,,
\end{align*}
where $\beta = \frac{\sqrt{2}}{(1-\gamma)^2\rho_{min}} \sqrt{\epsilon_{\text{stat}}  + \epsilon_{\text{bias}}}+\frac{\epsilon_{approx}}{1-\gamma}$, $\alpha_t = \left[1 - \eta\,\,C_t(1-\gamma)\right]$, $C_t := \min_s \{ \ppit(\tilde{\cA}_t(s) | s)  \, \Delta^t(s) \}$, $\tilde{\cA}_t(s) := \argmax_a \qpit(s, a)$, $\ppit(\tilde{\cA}_t(s) | s)= \sum_{a \in \tilde{\cA}_t(s)} \ppit(a_t(s) | s)$ and $\Delta^t(s) := \max_a \qpit(s, a) - \max_{a\notin \tilde{\cA}} \qpit(s, a)$.
\label{thm:spma-fa-inex-adv}
\end{restatable}

\begin{proof}
    Using~\cref{lem:inex-lin-improv} with $\pi = \bppitt$, 
\begin{align*}
          V^{*}(s) - V^{\bppitt}(s) \leq \left[1 - \eta\,\,C_t(1-\gamma)\right]\left[V^{*}(s) - V^{\pi_{t}}(s)\right]+ \frac{\epsilon_{approx}}{1-\gamma}
      \end{align*}  
The rest of the proof is similar to the proof of~\cref{thm:spma-fa}. For completeness, we repeat it here. By adding and removing $V^{\pi_{t+1}}(s)$ to both sides and rearranging we have 
\begin{align*}
    V^{*}(s) - V^{\pi_{t+1}}(s) & \leq  \left[1 - \eta\,\,C_t(1-\gamma)\right]\left[V^{*}(s) - V^{\pi_{t}}(s)\right] + V^{\bppitt}(s) - V^{\pi_{t+1}}(s)+ \frac{\epsilon_{approx}}{1-\gamma}.\\
\end{align*}
Taking the expectation w.r.t. $\rho$ we obtain: 
\begin{align*}
    J(\pi^*) - J(\pi_{t+1}) & \leq  \left[1 - \eta\,\,C_t(1-\gamma)\right]\left[J(\pi^*) - J(\pi_{t})\right] + \underbrace{J(\pi_{t+1/2}) - J(\pi_{t+1})}_{:=E_1}+ \frac{\epsilon_{approx}}{1-\gamma}.\\
\end{align*}
The term $E_1$ can be bounded as follows:  
\begin{align*}
    E_1 &=\sum_s \dbpitt(s) \inner{\qpitt(s,.)}{\bppitt(.|s)-\ppitt(.|s)}\\
    & \leq \sum_s \dbpitt(s) \norm{\qpitt(s,.)}_{\infty}\norm{\bppitt(.|s)-\ppitt(.|s)}_{1}\tag{Holder inequality}\\
    & \leq \frac{1}{1-\gamma}\sum_s \dbpitt(s)\norm{\bppitt(.|s)-\ppitt(.|s)}_{1} \tag{since $\norm{\qpitt(s,.)}_{\infty} \leq \frac{1}{1-\gamma}$}\\
    & =  \frac{1}{1-\gamma}\sum_s \frac{\dbpitt(s)}{\rho(s)}\rho(s) \norm{\bppitt(.|s)-\ppitt(.|s)}_{1}\\ 
    & \leq \frac{1}{1-\gamma} \norm{\frac{\dbpitt}{\rho}}_{\infty} \sum_s \rho(s) \norm{\bppitt(.|s)-\ppitt(.|s)}_{1}\\ 
   & \leq \frac{1}{(1-\gamma)\rho_{min}} \sum_s \rho(s) \norm{\bppitt(.|s)-\ppitt(.|s)}_{1}\tag{since $\dbpitt(s) \leq 1$ and using~\cref{as:exploration}}\\
   & \leq \frac{1}{(1-\gamma)^2\rho_{min}} \sum_s \dpit(s) \norm{\bppitt(.|s)-\ppitt(.|s)}_{1} \tag{sicne $\dpit \geq (1-\gamma) \rho$}\\
   & \leq \frac{\sqrt{2}}{(1-\gamma)^2\rho_{min}} \sum_s \dpit(s) \sqrt{ \text{KL}(\bppitt(.|s)\, || \, \ppitt(.|s))} \tag{using strong convexity of $ \text{KL}$ divergence or Pinsker's inequality}\\
   & \leq \frac{\sqrt{2}}{(1-\gamma)^2\rho_{min}} \sqrt{\underbrace{\sum_s \dpit(s)  \text{KL}(\bppitt(.|s)\, || \, \ppitt(.|s))}_{:=E_2}} \tag{due to concavity of $\sqrt{ }$}
\end{align*}
where $E_2$ can be bounded as follows: 
\begin{align*}
    E_2 &= \tilde{\ell}_t(\theta_{t+1}) \\
        &= \tilde{\ell}_t(\theta_{t+1}) - \min_{\theta}\tilde{\ell}_t(\theta_{t+1})+\min_{\theta}\tilde{\ell}_t(\theta_{t+1})\\
        & \leq \epsilon_{stat} +\min_{\theta}\tilde{\ell}_t(\theta_{t+1})\tag{using~\cref{as:samplingerror}}\\
        & \leq \epsilon_{stat} + \epsilon_{bias}\tag{using~\cref{as:bias}}
\end{align*}

Putting everything together we have: 
\begin{align*}
    E_1 \leq \underbrace{\frac{\sqrt{2}}{(1-\gamma)^2\rho_{min}} \sqrt{\epsilon_{\text{stat}}+ \epsilon_{\text{bias}}}}_{:=\beta'}
\end{align*}
Therefore we have 
\begin{align*}
    J(\pi^*) - J(\pi_{t+1}) & \leq  \underbrace{\left[1 - \eta\,\,C_t(1-\gamma)\right]}_{\alpha_t}\left[J(\pi^*) - J(\pi_{t})\right] + \underbrace{\beta'+\frac{\epsilon_{approx}}{1-\gamma}}_{:=\beta}.
\end{align*}
Unrolling the above recursion for $T$ iterations, 
\begin{align*}
    J(\pi^*) - J(\pi_{T}) \leq \left(\prod_{t=0}^{T-1} \alpha_t\right) (J(\pi^*) - J(\pi_0)) + \beta \sum_{t=0}^{T-1}\prod_{i=t+1}^{T-1}\alpha_i 
\end{align*}

\end{proof}


%% file: appendix-mdp-results.tex
\vspace{-2ex}
\section{Tabular MDP Experiments}
\label{app:tabular-mdp-experiments}
In this section, we empirically evaluate $\Alg$ on tabular MDP environments. For these experiments, we use Cliff World~\citep{sutton2018reinforcement} and Frozen Lake~\citep{brockman2016openai} following the setup in~\cite{vaswani2024decision}. In subsection,~\ref{tabular-mdp-exp-softmax-param} we examine the case where the policy is parametrized using softmax tabular representation. In subsection,~\ref{tabular-mdp-exp-linear-param} we investigate the function approximation setting as described in Section~\ref{sec:fa}, where the policy is parametrized using a linear model.

\subsection{Softmax Tabular Representation}
\label{tabular-mdp-exp-softmax-param}
For this parametrization we initialize $z \in \mathbb{R}^{S \times A}$ uniformly , i.e., $\pi_0(a|s) = \frac{1}{|\mathcal{A}|}$ for each $a$ and $s$. Furthermore, for each algorithm, we set $\eta$ using a grid search and pick the step-sizes that result in the best area under the curve (AUC). The tabular MDP results suggest $\Alg$ and $\NPG$ achieve similar performance and they both outperform $\SPG$~\citep{sutton1999policy, schulman2017proximal} (see Fig.~\ref{fig:mdp-tabular}). To analyze the sensitivity of each algorithm to the choice of $\eta$, we examine each optimizer across different values of $\eta$. The results in Fig.~\ref{fig:mdp-tabular-sensitivity} suggest that overall $\SPG$ (in green) is more sensitive to different values of $\eta$ compared to $\Alg$ (blue) and $\NPG$ (red).

\begin{figure}[ht]
\centering
\includegraphics[width=\textwidth]{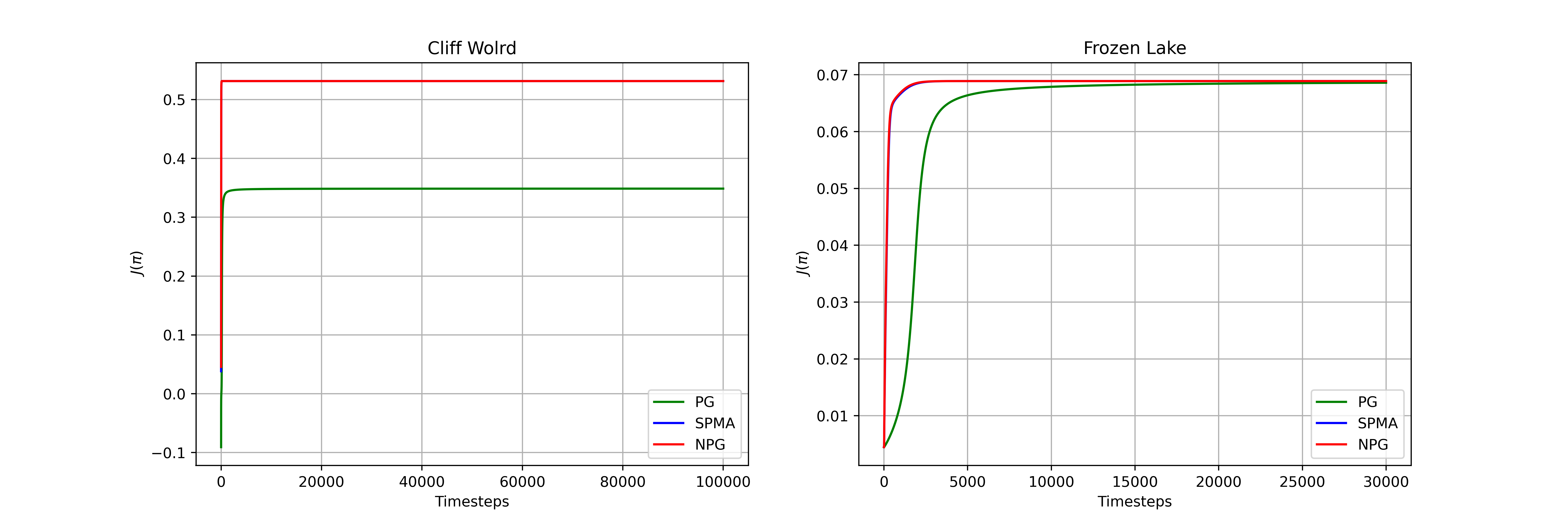}
\caption{$\Alg$ matches the performance of $\NPG$ and they both outperform $\SPG$.}
\label{fig:mdp-tabular}
\end{figure}

\begin{figure}[ht]
\centering
\includegraphics[width=\textwidth]{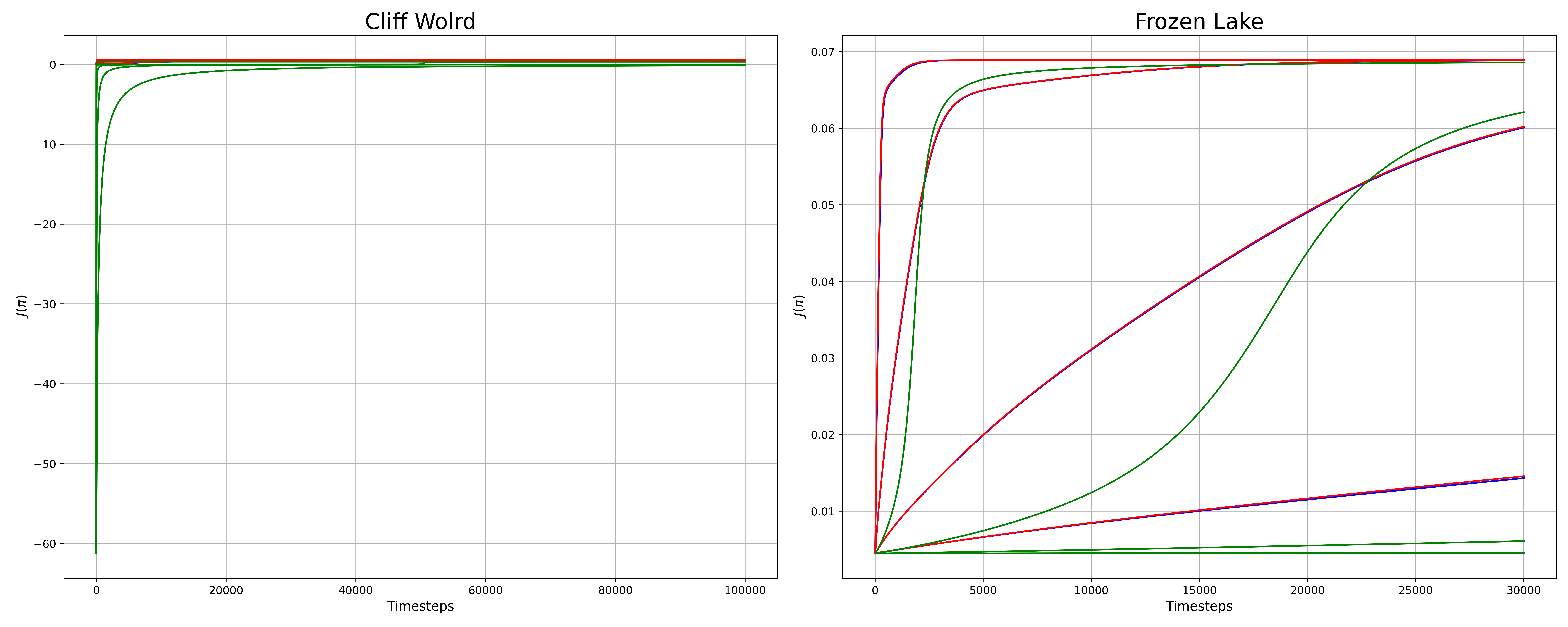}
\caption{$\SPG$ (green) is more sensitive to $\eta$ compared to $\Alg$ and $\NPG$ (blue and red).}
\label{fig:mdp-tabular-sensitivity}
\end{figure}

\clearpage
\subsection{Linear Functional Approximation (Linear FA)}
\label{tabular-mdp-exp-linear-param}
For the Linear FA setting, we use a log-linear policy parametrization: $\pi_t(a | s) = \frac{\exp(\X(s, a)\theta)}{\sum_{a'}\exp(\X(s, a')\theta)}$, with $\X \in \mathbb{R}^{\cS \cA \times d}$ and $\theta \in \mathbb{R}^d$ representing the features and parameters. We use constant initialization for $\theta$ and following~\citet{vaswani2024decision}, use tile-coded features for $\X$. As in the previous section, we set $\eta$ for $\Alg$ and $\MDPO$ via grid search and report results based on the best AUC. For the inner loop optimization (e.g., minimizing~\cref{eq:approx-surrogate}), we use Armijo line search~\citep{armijo1966minimization}, avoiding an additional grid search for the step-size. For $\SPG$ we use the update from~\citet{mei2020global} where Armijo line search is used to set $\eta$. 

We make the following observations from the results: (i) $\SPG$ performs poorly in the linear FA setting, while both $\Alg$ and $\MDPO$ perform well when the parameter dimension $d$ and the number of inner loop optimizations $m$ are sufficiently large. (ii) In the CW environment, for smaller $d$, $\Alg$ converges faster than $\MDPO$ (Fig.~\ref{fig:cw-m-25-50-lfa}, top row). Increasing $m$ from 25 to 50 narrows the gap between $\Alg$ and $\MDPO$ (top vs. bottom row). (iii) In the FL environment, $\Alg$ and $\MDPO$ perform similarly, both outperforming $\SPG$ (Fig.~\ref{fig:fl-m-25-50-lfa}).

\begin{figure}[!ht]
\centering
\begin{minipage}[b]{\textwidth}
    \centering
    \includegraphics[width=0.84\textwidth]{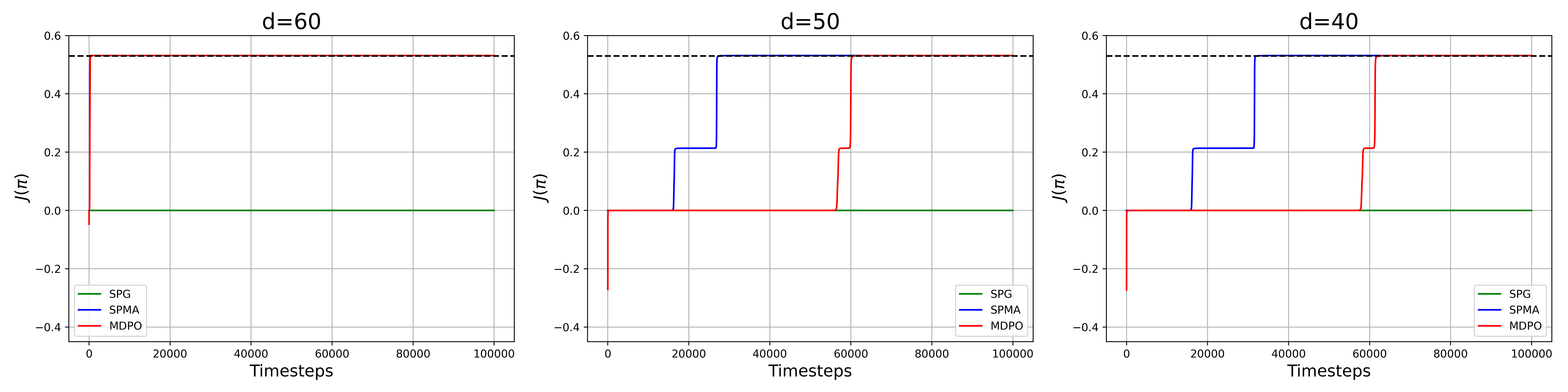}
\end{minipage}
\begin{minipage}[b]{\textwidth}
    \centering
    \includegraphics[width=0.84\textwidth]{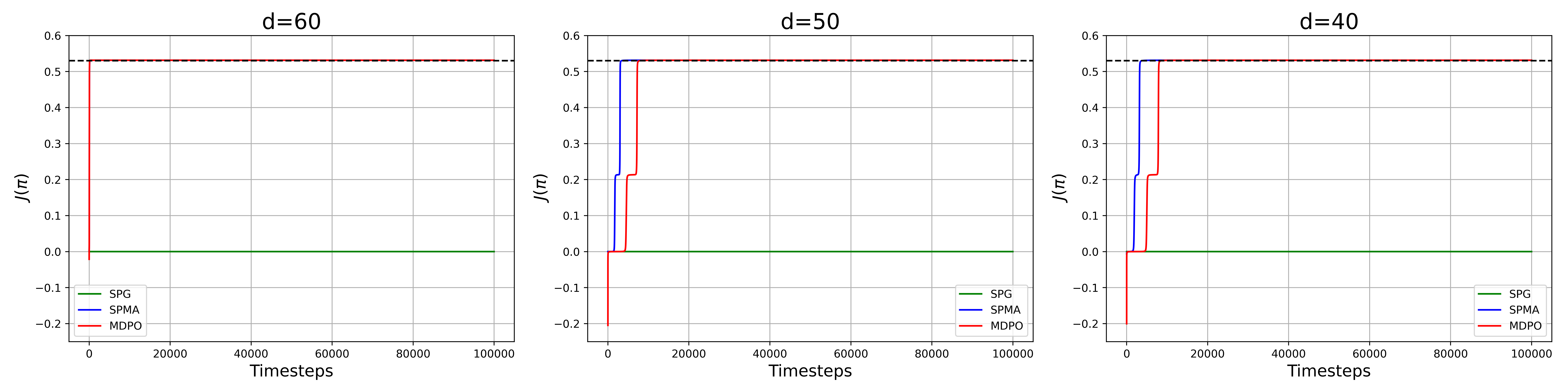}
\end{minipage}
\caption{CW environment: The top row ($m = 25$) shows that $\Alg$ converges faster than $\MDPO$ as $d$ decreases, while the bottom row ($m = 50$) shows the gap decreases as the number of inner loop optimizations increases.}
\label{fig:cw-m-25-50-lfa}
\end{figure}
\begin{figure}[!ht]
\centering
\begin{minipage}[b]{\textwidth}
    \centering
    \includegraphics[width=0.84\textwidth]{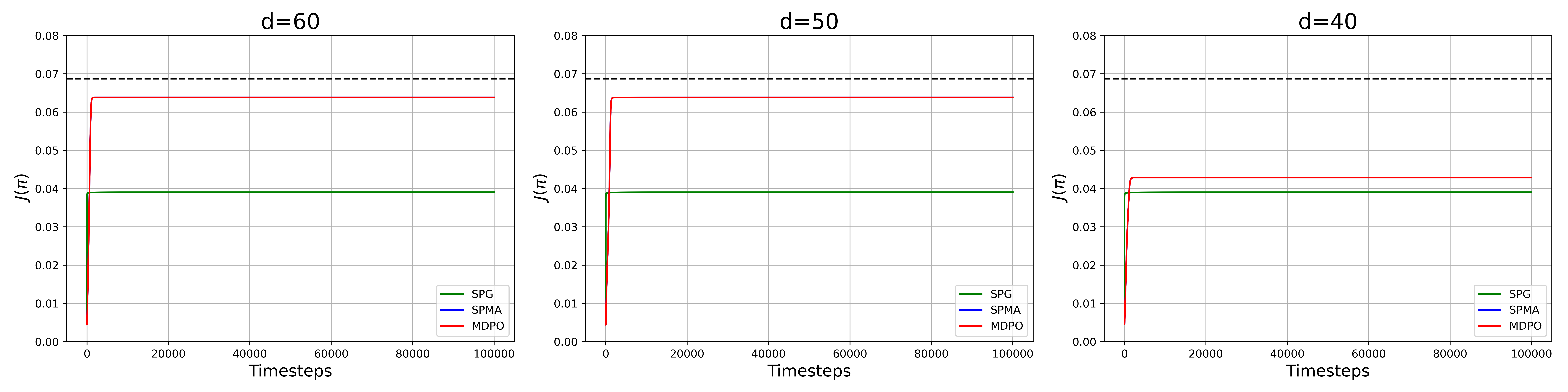}
\end{minipage}
\begin{minipage}[b]{\textwidth}
    \centering
    \includegraphics[width=0.84\textwidth]{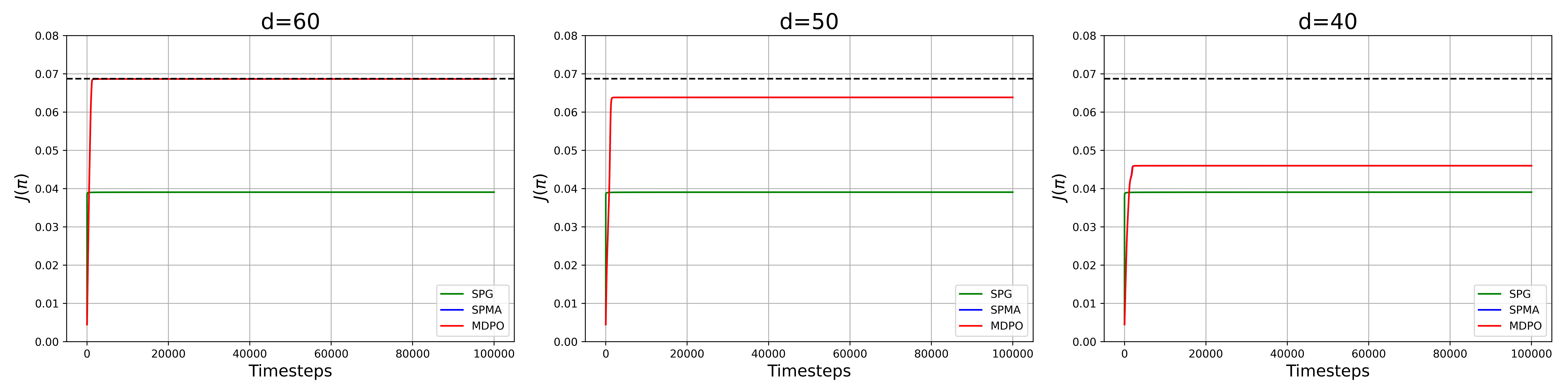}
\end{minipage}
\caption{FL environment: The top row ($m = 25$) and bottom row ($m = 50$) show that $\Alg$ and $\MDPO$ have similar convergence and both outperform $\SPG$. The performance of both $\Alg$ and $\MDPO$ improves as $d$ increases (i.e., the bias decreases) and $m$ increases (i.e., the optimization error decreases).}
\label{fig:fl-m-25-50-lfa}
\end{figure}






\clearpage
\subsection{\texorpdfstring{Empirical Verification that $C_t$ is Non-zero}{Empirical Verification that C\_t is Bounded Away from Zero}}
\label{app:tabular-mdp-exp-ct}
 In this section, we empirically investigate the value of $C_t$. Recall that $C_t := \min_s \{ \ppit(\tilde{\cA}_t(s) | s)  \, \Delta^t(s) \}$, $\tilde{\cA}_t(s) := \argmax_a \qpit(s, a)$, $\ppit(\tilde{\cA}_t(s) | s) = \sum_{a \in \tilde{\cA}_t(s)} \ppit(a | s)$, and $\Delta^t(s) := \max_a \qpit(s, a) - \max_{a \notin \tilde{\cA}_t(s)} \qpit(s, a)$. To this end, we compute the per-state metric $C_t(s) := \ppit(\tilde{\cA}_t(s) | s)  \, \Delta^t(s)$ and its components, $\ppit(\tilde{\cA}_t(s) | s)$ and $\Delta^t(s)$, for the CliffWorld and FrozenLake environments using softmax tabular policy parametrization. Our results indicate that $C_t(s)$ is nonzero for all states except for certain terminal states, specifically the goal state and the hole states in FrozenLake, where $\qpi(s, a)$ is zero for all actions $a$ and policies $\pi$. In these cases, we observe that $\Delta^t(s)$ becomes zero, leading to $C_t(s) = 0$ and consequently $C_t = 0$. Note that in the CliffWorld environment, the transition probability matrix moves the agent back to the starting point at terminal states, so they do not pose an issue. However, in the FrozenLake environment, the agent remains at the terminal state with probability one.

Given these empirical results, one could modify the theorems in this paper to exclude the problematic terminal states. For example, in~\cref{thm:smdpo-constant-mdp}, instead of bounding $V^{\pi^{*}}(s) - V^{\pi_{T}}(s)$ for all states using the $\ell_{\infty}$ norm, we could exclude states where $Q^{\pi^{*}}(s, a) = 0$ for all actions $a$. This exclusion is analogous to the bandit literature, which considers only arms with a nonzero gap in the regret bound. In~\cref{fig:CW-FL-ct-tab}, we plot $C_t$ using softmax tabular policy parameterizations for both CliffWorld and FrozenLake environments, confirming that after excluding terminal states, $C_t$ is lower-bounded by a positive constant.  

\begin{figure}[ht]
\centering
\includegraphics[width=\textwidth]{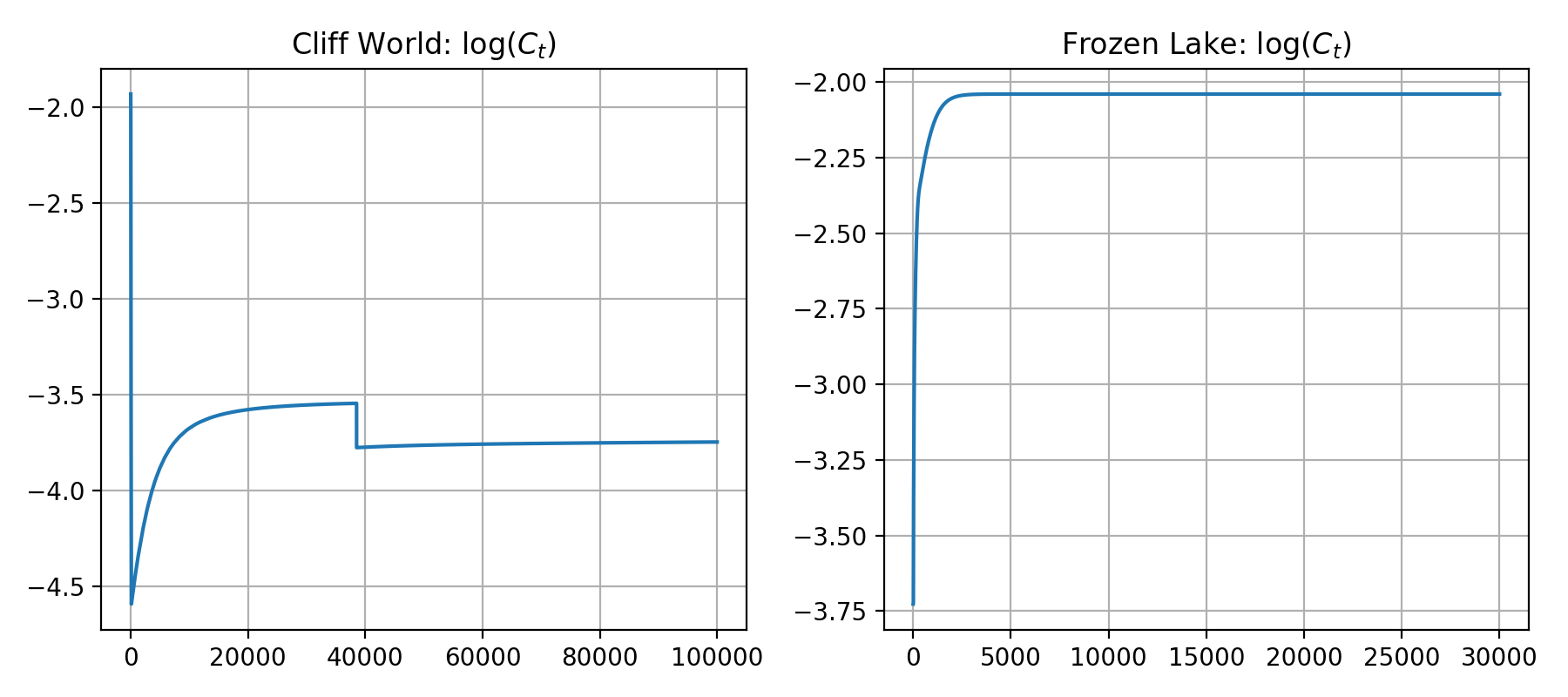}
\caption{$\Alg$ with softmax tabular policy parametrization: After excluding terminal states, results confirm that $C_t$ is lower-bounded by a positive constant.}
\label{fig:CW-FL-ct-tab}
\end{figure}

%% file: appendix-stable-baseline-details.tex
\vspace{-2ex}
\clearpage
\section{Additional Details for Stable Baselines Experiments}
\vspace{-2ex}
\label{app:stable-baseline-experiment-details}
In subsection~\ref{subsec:atari-mujoco-details}, we provide additional details on the hyper-parameters used for the results in~\cref{sec:experiments}. Next, we present an ablation study on the number of inner loop optimization steps ($m$) in subsection~\ref{subsec:ablation-m}.
\vspace{-2ex}
\subsection{Atari and Mujoco Details}
\vspace{-2ex}
\label{subsec:atari-mujoco-details}
In the Atari experiments, we use the default hyper-parameters for each method from stable baselines~\citep{stable-baselines3}. This choice is motivated by two factors: i) following the work of~\citet{tomar2020mirror}, we aim to evaluate the effectiveness of different surrogate losses without conducting an exhaustive search over numerous hyper-parameters; ii) the CNN-based actor and critic networks make grid searching over many hyper-parameters (e.g., framestack, GAE $\lambda$, horizon length, discount factor) computationally infeasible. For a complete list of hyper-parameters used in the Atari experiments, see~\cref{tab:atari-hyperparams}.

In the MuJoCo experiments, we use the default hyper-parameters from stable baselines for each method, but perform a grid search on the Adam inner loop step size for $\PPO$ and $\TRPOC$ (best among $[3 \times 10^{-3}, 3 \times 10^{-4}, 3 \times 10^{-5}]$) and the probability ratio clipping parameter in $\PPO$ (best from $[0.1, 0.2, 0.3]$). For the regularized surrogates (i.e., the remaining methods: $\Alg$, $\MDPO$, and $\TRPOR$), we avoid a grid search for the inner loop step size by using a full batch (i.e., the horizon length) along with the Armijo line search~\citep{armijo1966minimization}. See~\cref{tab:mujoco-hyperparams} for the full list of hyper-parameters used in the MuJoCo experiments.

To set $\eta$ for the regularized surrogates, we perform a grid search over five fixed values ($[0.3, 0.5, 0.7, 0.9, 1.0]$). Although~\citet{tomar2020mirror} anneals $\eta$ from 1 to 0 during training, we observe that using a constant step size results in better performance. Our grid search strategy for all stable baselines experiments is consistent: we run the experiments for 2 million iterations, select the hyper-parameters that yield the best AUC, and then use these hyper-parameters for an additional 8 million iterations.

\begin{table}[ht]
\centering
\resizebox{0.80\textwidth}{!}{
\begin{tabular}{|l|c|c|c|c|c|c|c|}
\hline
\textbf{Hyperparameter} & \textbf{SPMA} & \textbf{MDPO} & \textbf{TRPO\_regularized} & \textbf{TRPO\_constrained} & \textbf{PPO} \\ 
\hline
Reward normalization & \xmark & \xmark & \xmark & \xmark & \xmark \\
Observation normalization & \xmark & \xmark & \xmark & \xmark & \xmark \\
Orthogonal weight initialization & \cmark & \cmark & \cmark & \cmark & \cmark \\
Value function clipping & \xmark & \xmark & \xmark & \xmark & \xmark \\
Gradient clipping & \xmark & \xmark & \xmark & \xmark & \cmark \\
Probability ratio clipping & \xmark & \xmark & \xmark & \xmark & \cmark \\
\hline
Adam step-size & \multicolumn{6}{c|}{$3\times10^{-4}$} \\
Minibatch size & \multicolumn{6}{c|}{256} \\
Framestack & \multicolumn{6}{c|}{4} \\
Number of environment copies & \multicolumn{6}{c|}{8} \\
GAE $\lambda$ & \multicolumn{6}{c|}{0.95} \\
Horizon (T) & \multicolumn{6}{c|}{128} \\
Number of inner loop updates (m) & \multicolumn{6}{c|}{5} \\
Entropy coefficient & \multicolumn{6}{c|}{0} \\
Discount factor & \multicolumn{6}{c|}{0.99} \\
Total number of timesteps & \multicolumn{6}{c|}{$10^7$} \\
Number of runs for plot averages & \multicolumn{6}{c|}{5} \\
Confidence interval for plot runs & \multicolumn{6}{c|}{$\sim 95\%$} \\
\hline
\end{tabular}
}
\caption{Hyper-parameters for Atari experiments.}
\label{tab:atari-hyperparams}
\end{table}


\begin{table}[ht]
\centering
\resizebox{0.80\textwidth}{!}{
\begin{tabular}{|l|c|c|c|c|c|c|c}
\hline
\textbf{Hyperparameter} & \textbf{SPMA} & \textbf{MDPO} & \textbf{TRPO\_regularized} & \textbf{TRPO\_constrained} & \textbf{PPO} \\ 
\hline
Minibatch size & 2048 & 2048 & 2048 & 64 & 64 \\
Reward normalization & \xmark & \xmark & \xmark & \xmark & \xmark \\
Observation normalization & \xmark & \xmark & \xmark & \xmark & \xmark \\
Orthogonal weight initialization & \cmark & \cmark & \cmark & \cmark & \cmark \\
Value function clipping & \xmark & \xmark & \xmark & \xmark & \xmark \\
Gradient clipping & \xmark & \xmark & \xmark & \xmark & \cmark \\
Probability ratio clipping & \xmark & \xmark & \xmark & \xmark & \cmark \\
Adam step-size & \xmark & \xmark & \xmark & \cmark & \cmark \\
\hline
GAE $\lambda$ & \multicolumn{6}{c|}{0.95} \\
Horizon (T) & \multicolumn{6}{c|}{2048} \\
Number of inner loop updates (m) & \multicolumn{6}{c|}{5} \\
Entropy coefficient & \multicolumn{6}{c|}{0} \\
Discount factor & \multicolumn{6}{c|}{0.99} \\
Total number of timesteps & \multicolumn{6}{c|}{$10^7$} \\
Number of runs for plot averages & \multicolumn{6}{c|}{5} \\
Confidence interval for plot runs & \multicolumn{6}{c|}{$\sim 95\%$} \\
\hline
\end{tabular}
}
\caption{Hyper-parameters for MuJoCo experiments.}
\label{tab:mujoco-hyperparams}
\end{table}

\subsection{Ablation Study on the Number of Inner Loop Optimization Steps}
\label{subsec:ablation-m}
In this subsection, we investigate the effect of varying the number of inner loop optimization steps ($m$) in the stable baselines experiments. Consistent with~\citet{tomar2020mirror}, we observe that using $m=1$ results in poor performance, so we focus on larger values of $m$. In the MuJoCo experiments, increasing $m$ from 5 to 10 and 15 consistently improves the performance of $\Alg$ (see~\cref{fig:mujoco-m-ablation}). Specifically, for larger $m$, $\Alg$ becomes comparable to $\TRPOC$ on Hopper and Ant, while outperforming it on HalfCheetah (see~\cref{fig:mujoco-m-vertical}). 

For the Atari experiments, we observe that increasing $m$ does not necessarily improve the results across methods (see~\cref{fig:atari-m-ablation}). We conjecture that this is a side-effect of using a constant tuned step-size (for $m = 5$) in the inner-loop. In the future, we plan to run the full grid-search for the inner-loop step-size for each value of $m$. Alternatively, we plan to investigate an adaptive way of setting the inner-loop step-size. 

\begin{figure}[ht]
\centering
\begin{minipage}[b]{\textwidth}
    \centering
    \includegraphics[width=0.73\textwidth]{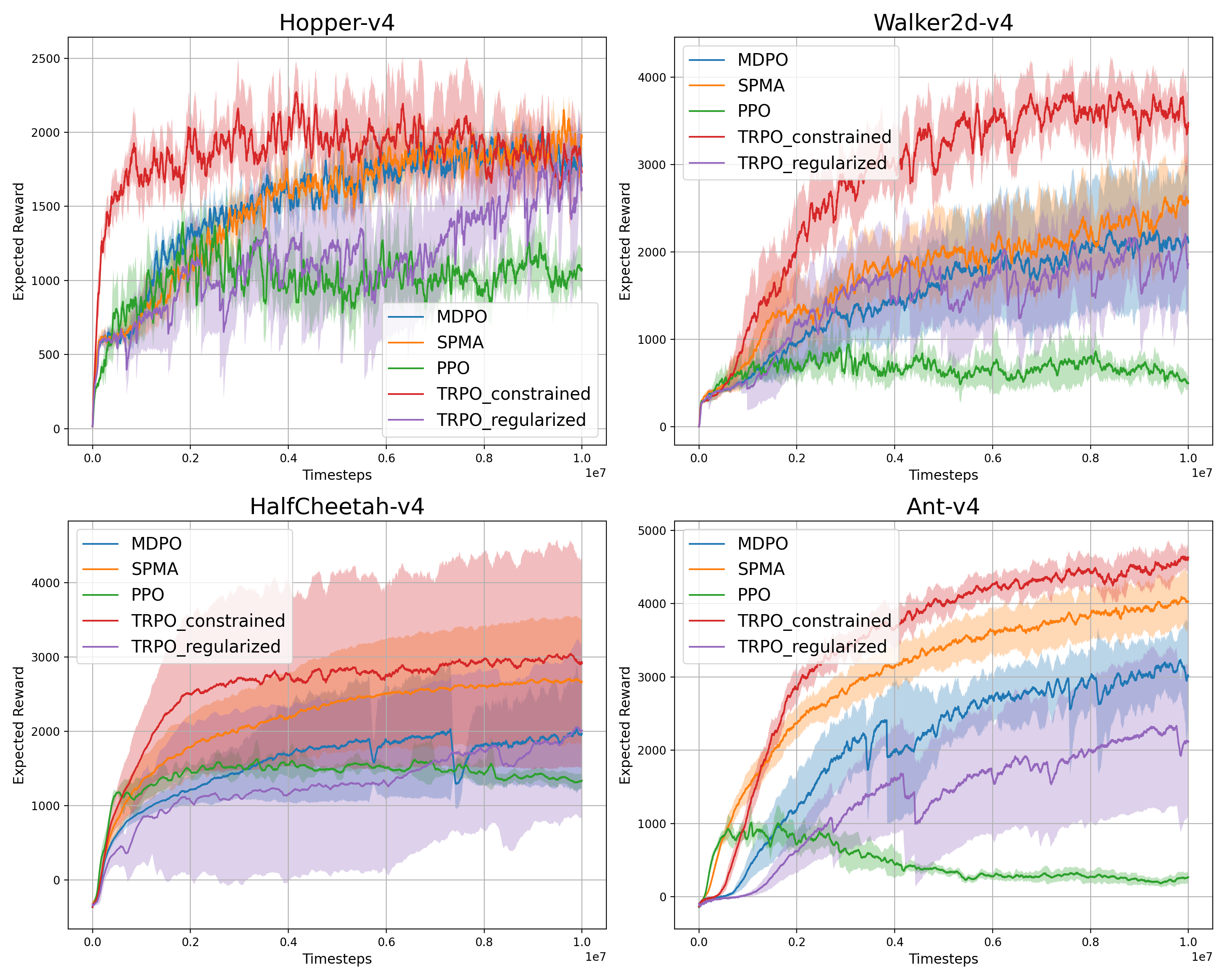}
    \caption*{(a)}
\end{minipage}

\vspace{0.65cm}

\begin{minipage}[b]{\textwidth}
    \centering
    \includegraphics[width=0.73\textwidth]{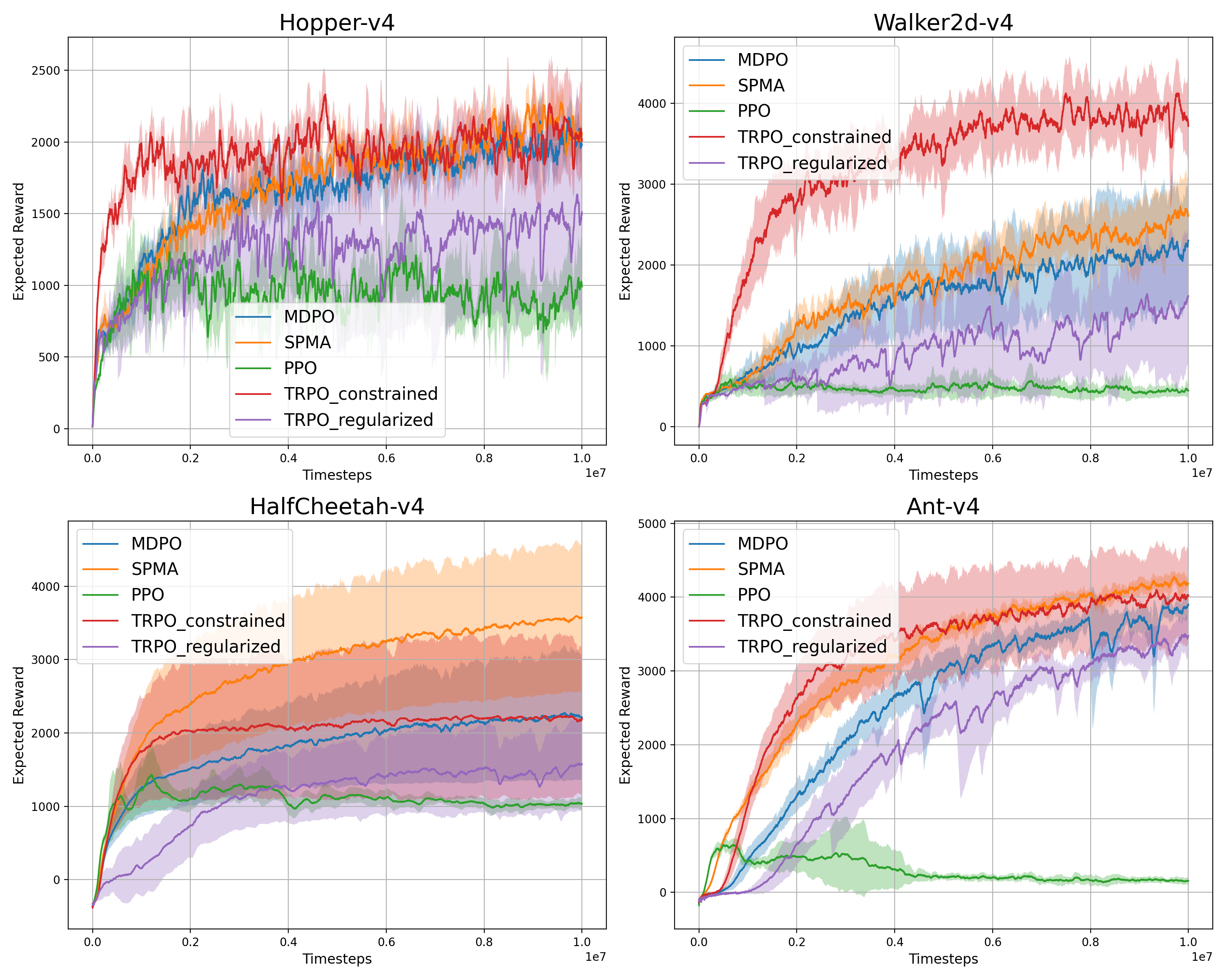}
    \caption*{(b)}
\end{minipage}

\caption{MuJoCo results for $m=10$ (a) and $m=15$ (b). As $m$ increases from 5 (~\cref{fig:mdp-mujoco}) to 10 and 15, $\Alg$ shows performance comparable to the fine-tuned $\TRPOC$.}
\label{fig:mujoco-m-vertical}
\end{figure}

\begin{figure}[ht]
\centering
\begin{minipage}[b]{\textwidth}
    \centering
    \includegraphics[width=0.73\textwidth]{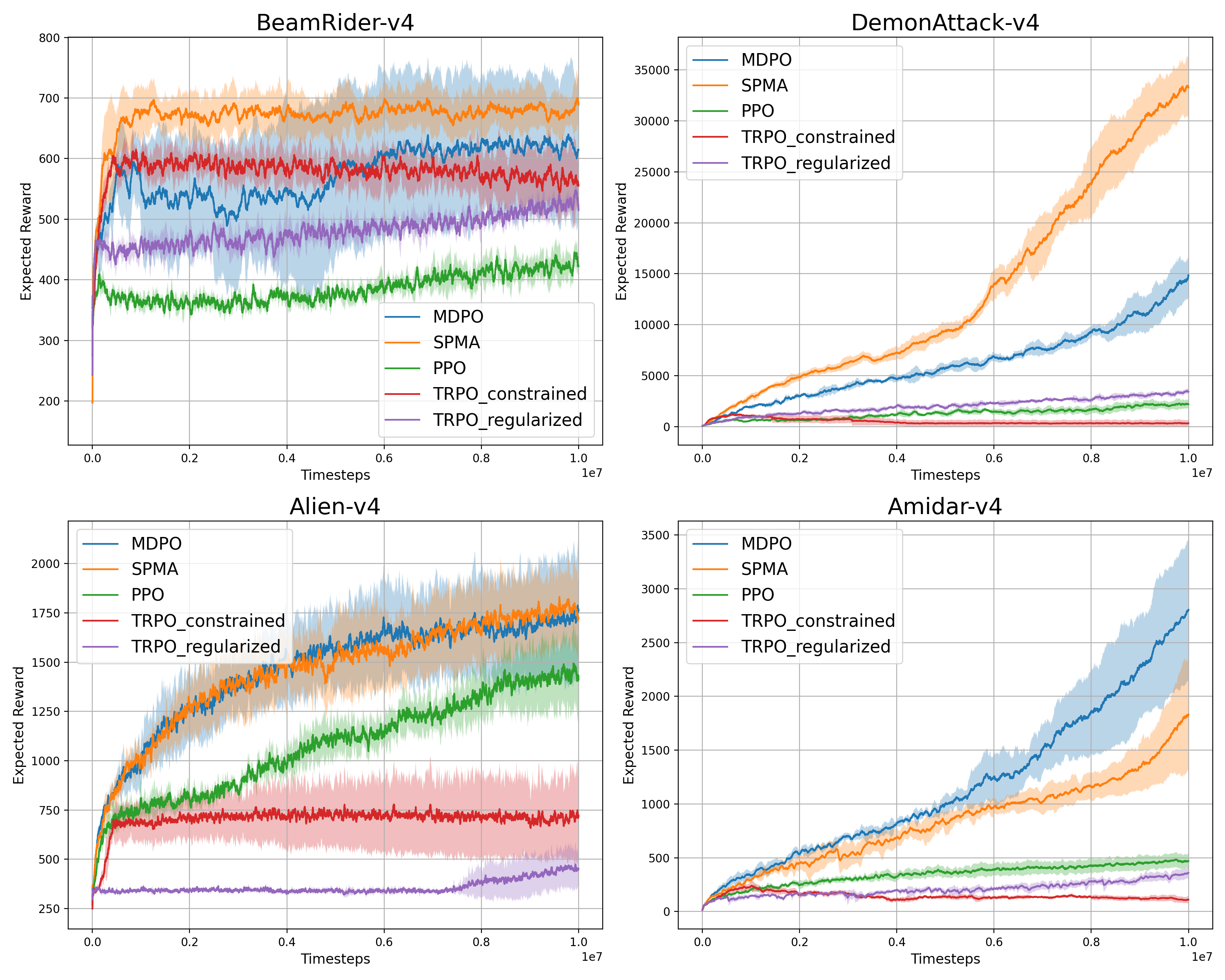}
    \caption*{(a)}
\end{minipage}

\vspace{0.65cm}

\begin{minipage}[b]{\textwidth}
    \centering
    \includegraphics[width=0.73\textwidth]{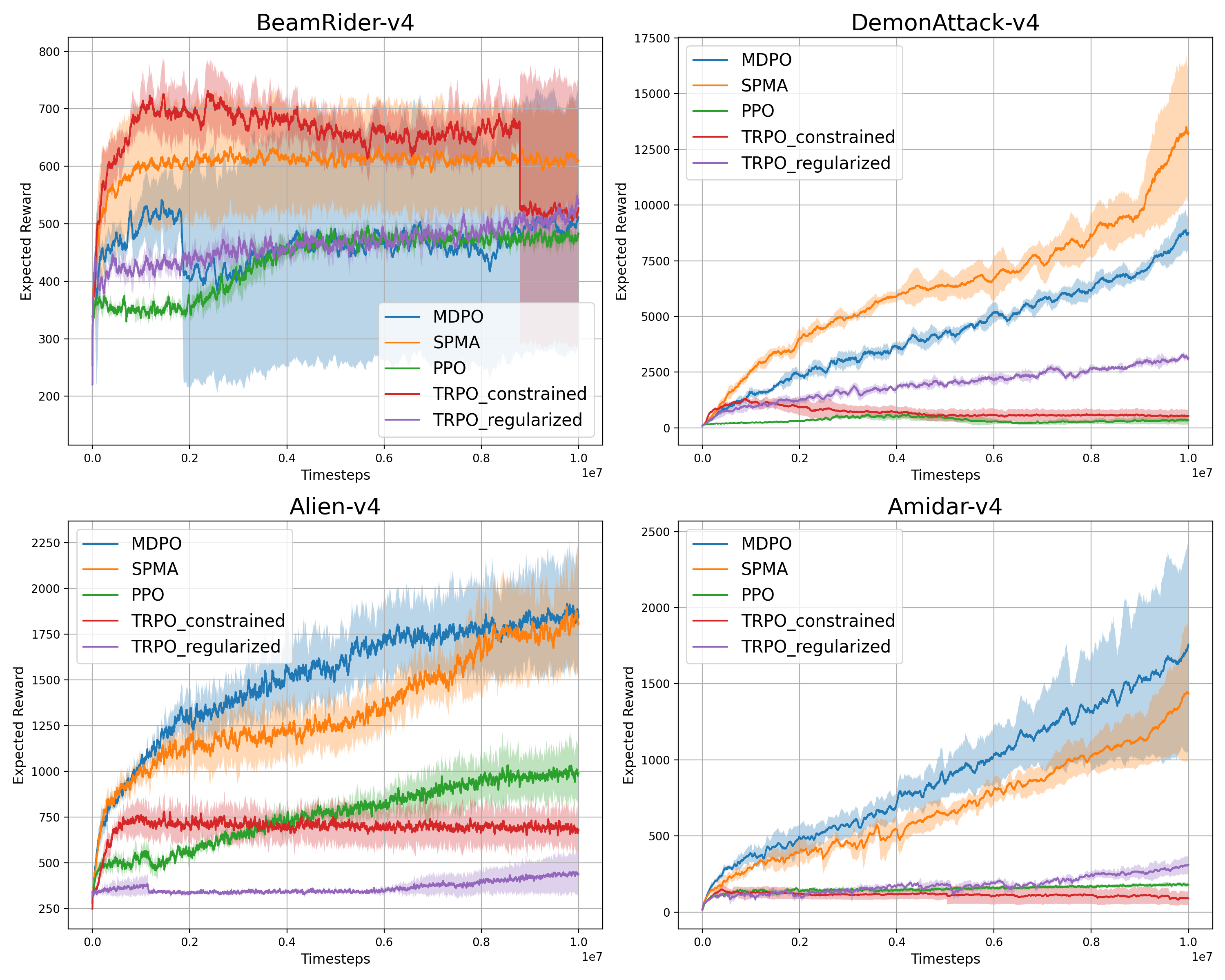}
    \caption*{(b)}
\end{minipage}

\caption{Atari results for $m=10$ (top) and $m=15$ (bottom). Increasing $m$ does not necessarily lead to performance improvements.}
\label{fig:atari-m-vertical}
\end{figure}
\begin{figure}[ht]
\centering
\includegraphics[width=\textwidth]{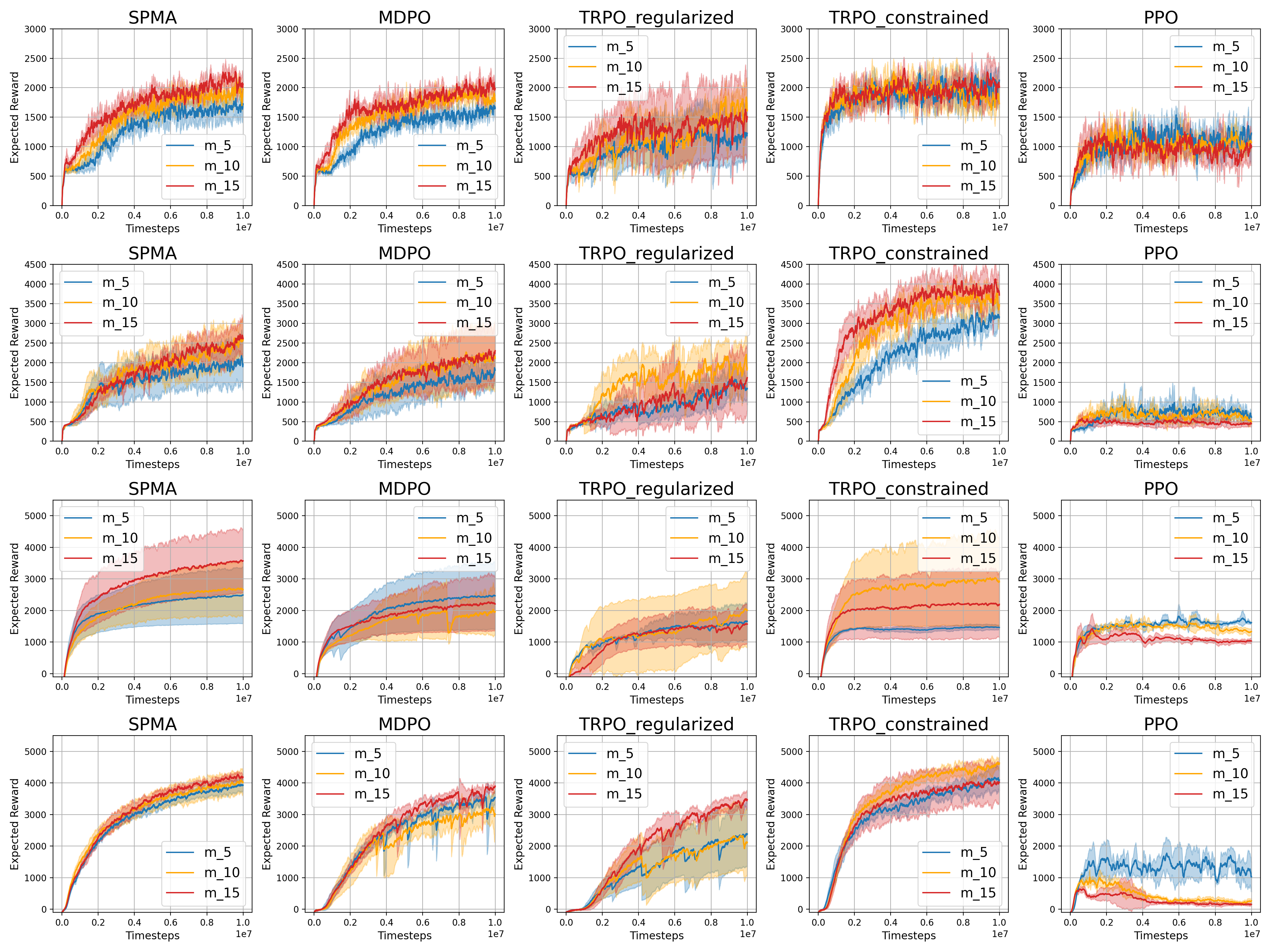}
\caption{MuJoCo ablation on $m$: The rows correspond to the Hopper-v4, Walker2d-v4, HalfCheetah-v4, and Ant-v4 environments, respectively. As the number of inner loop optimization steps $m$ increases, $\Alg$ shows improvements in expected reward and becomes comparable to the fine-tuned $\TRPOC$.}
\label{fig:mujoco-m-ablation}
\end{figure}

\begin{figure}[ht]
\centering
\includegraphics[width=\textwidth]{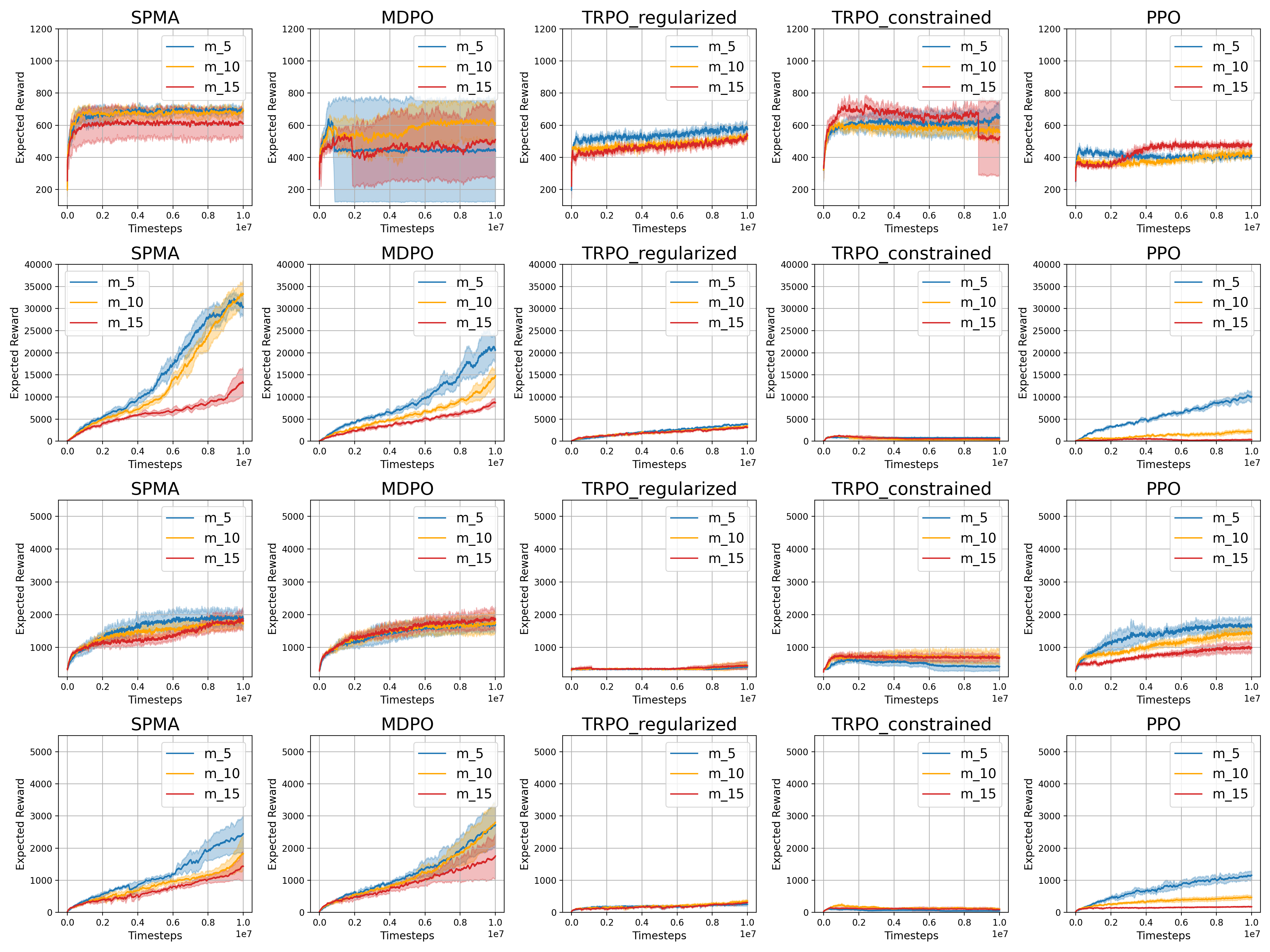}
\caption{Atari ablation on $m$: The rows correspond to the BeamRider-v4, DemonAttack-v4, Alien-v4, and Amidar-v4 games. We observe that increasing $m$ does not necessarily improve results across methods.}
\label{fig:atari-m-ablation}
\end{figure}

%% file: main.bbl
\begin{thebibliography}{}

\bibitem[Agarwal et~al., 2021]{agarwal2021theory}
Agarwal, A., Kakade, S.~M., Lee, J.~D., and Mahajan, G. (2021).
\newblock On the theory of policy gradient methods: Optimality, approximation, and distribution shift.
\newblock {\em J. Mach. Learn. Res.}, 22(98):1--76.

\bibitem[Alfano and Rebeschini, 2022]{alfano2022linear}
Alfano, C. and Rebeschini, P. (2022).
\newblock Linear convergence for natural policy gradient with log-linear policy parametrization.
\newblock {\em arXiv preprint arXiv:2209.15382}.

\bibitem[Armijo, 1966]{armijo1966minimization}
Armijo, L. (1966).
\newblock Minimization of functions having lipschitz continuous first partial derivatives.
\newblock {\em Pacific Journal of mathematics}, 16(1):1--3.

\bibitem[Beck and Teboulle, 2003]{beck2003mirror}
Beck, A. and Teboulle, M. (2003).
\newblock Mirror descent and nonlinear projected subgradient methods for convex optimization.
\newblock {\em Operations Research Letters}, 31(3):167--175.

\bibitem[Bellemare et~al., 2013]{bellemare2013arcade}
Bellemare, M.~G., Naddaf, Y., Veness, J., and Bowling, M. (2013).
\newblock The arcade learning environment: An evaluation platform for general agents.
\newblock {\em Journal of Artificial Intelligence Research}, 47:253--279.

\bibitem[Bhandari and Russo, 2021]{bhandari2021linear}
Bhandari, J. and Russo, D. (2021).
\newblock On the linear convergence of policy gradient methods for finite mdps.
\newblock In {\em International Conference on Artificial Intelligence and Statistics}, pages 2386--2394. PMLR.

\bibitem[Brockman, 2016]{brockman2016openai}
Brockman, G. (2016).
\newblock Openai gym.
\newblock {\em arXiv preprint arXiv:1606.01540}.

\bibitem[Bubeck et~al., 2015]{bubeck2015convex}
Bubeck, S. et~al. (2015).
\newblock Convex optimization: Algorithms and complexity.
\newblock {\em Foundations and Trends{\textregistered} in Machine Learning}, 8(3-4):231--357.

\bibitem[Cesa-Bianchi et~al., 2007]{cesa2007improved}
Cesa-Bianchi, N., Mansour, Y., and Stoltz, G. (2007).
\newblock Improved second-order bounds for prediction with expert advice.
\newblock {\em Machine Learning}, 66:321--352.

\bibitem[Chen et~al., 2023]{chenaccelerated}
Chen, Y.-J., Huang, N.-C., Lee, C.-p., and Hsieh, P.-C. (2023).
\newblock Accelerated policy gradient: On the convergence rates of the nesterov momentum for reinforcement learning.
\newblock In {\em Forty-first International Conference on Machine Learning}.

\bibitem[Engstrom et~al., 2019]{engstrom2019implementation}
Engstrom, L., Ilyas, A., Santurkar, S., Tsipras, D., Janoos, F., Rudolph, L., and Madry, A. (2019).
\newblock Implementation matters in deep rl: A case study on ppo and trpo.
\newblock In {\em International conference on learning representations}.

\bibitem[Freund and Schapire, 1997]{freund1997decision}
Freund, Y. and Schapire, R.~E. (1997).
\newblock A decision-theoretic generalization of on-line learning and an application to boosting.
\newblock {\em Journal of computer and system sciences}, 55(1):119--139.

\bibitem[Haarnoja et~al., 2017]{haarnoja2017reinforcement}
Haarnoja, T., Tang, H., Abbeel, P., and Levine, S. (2017).
\newblock Reinforcement learning with deep energy-based policies.
\newblock In {\em International conference on machine learning}, pages 1352--1361. PMLR.

\bibitem[Johnson et~al., 2023]{johnson2023optimal}
Johnson, E., Pike-Burke, C., and Rebeschini, P. (2023).
\newblock Optimal convergence rate for exact policy mirror descent in discounted markov decision processes.
\newblock {\em arXiv preprint arXiv:2302.11381}.

\bibitem[Kakade, 2001]{kakade2001natural}
Kakade, S.~M. (2001).
\newblock A natural policy gradient.
\newblock {\em Advances in neural information processing systems}, 14.

\bibitem[Khodadadian et~al., 2021]{khodadadian2021linear}
Khodadadian, S., Jhunjhunwala, P.~R., Varma, S.~M., and Maguluri, S.~T. (2021).
\newblock On the linear convergence of natural policy gradient algorithm.
\newblock In {\em 2021 60th IEEE Conference on Decision and Control (CDC)}, pages 3794--3799. IEEE.

\bibitem[Konda and Tsitsiklis, 2000]{konda2000actor}
Konda, V.~R. and Tsitsiklis, J.~N. (2000).
\newblock Actor-critic algorithms.
\newblock In {\em Advances in neural information processing systems}, pages 1008--1014.

\bibitem[Kuba et~al., 2022]{kuba2022mirror}
Kuba, J.~G., de~Witt, C.~S., and Foerster, J. (2022).
\newblock Mirror learning: A unifying framework of policy optimisation.
\newblock {\em arXiv preprint arXiv:2201.02373}.

\bibitem[Lan, 2023]{lan2023policy}
Lan, G. (2023).
\newblock Policy mirror descent for reinforcement learning: Linear convergence, new sampling complexity, and generalized problem classes.
\newblock {\em Mathematical programming}, 198(1):1059--1106.

\bibitem[Lavington et~al., 2023]{lavington2023target}
Lavington, J.~W., Vaswani, S., Babanezhad, R., Schmidt, M., and Roux, N.~L. (2023).
\newblock Target-based surrogates for stochastic optimization.
\newblock {\em arXiv preprint arXiv:2302.02607}.

\bibitem[Lazi{\'c} et~al., 2021]{lazic2021optimization}
Lazi{\'c}, N., Hao, B., Abbasi-Yadkori, Y., Schuurmans, D., and Szepesv{\'a}ri, C. (2021).
\newblock Optimization issues in kl-constrained approximate policy iteration.
\newblock {\em arXiv preprint arXiv:2102.06234}.

\bibitem[Lei and Ying, 2021]{lei2021sharper}
Lei, Y. and Ying, Y. (2021).
\newblock Sharper generalization bounds for learning with gradient-dominated objective functions.
\newblock In {\em International Conference on Learning Representations}.

\bibitem[Li et~al., 2021]{li2021softmax}
Li, G., Wei, Y., Chi, Y., Gu, Y., and Chen, Y. (2021).
\newblock Softmax policy gradient methods can take exponential time to converge.
\newblock In {\em Conference on Learning Theory}, pages 3107--3110. PMLR.

\bibitem[Liu et~al., 1906]{liu1906neural}
Liu, B., Cai, Q., Yang, Z., and Wang, Z. (1906).
\newblock Neural proximal/trust region policy optimization attains globally optimal policy (2019).
\newblock {\em arXiv preprint arXiv:1906.10306}.

\bibitem[Liu et~al., 2022]{liu2022loss}
Liu, C., Zhu, L., and Belkin, M. (2022).
\newblock Loss landscapes and optimization in over-parameterized non-linear systems and neural networks.
\newblock {\em Applied and Computational Harmonic Analysis}, 59:85--116.

\bibitem[Liu et~al., 2024]{liu2024elementary}
Liu, J., Li, W., and Wei, K. (2024).
\newblock Elementary analysis of policy gradient methods.
\newblock {\em arXiv preprint arXiv:2404.03372}.

\bibitem[Lu et~al., 2024]{lu2024towards}
Lu, M., Aghaei, M., Raj, A., and Vaswani, S. (2024).
\newblock Towards principled, practical policy gradient for bandits and tabular mdps.
\newblock {\em arXiv preprint arXiv:2405.13136}.

\bibitem[Mei et~al., 2021a]{mei2021understanding}
Mei, J., Dai, B., Xiao, C., Szepesvari, C., and Schuurmans, D. (2021a).
\newblock Understanding the effect of stochasticity in policy optimization.
\newblock {\em Advances in Neural Information Processing Systems}, 34:19339--19351.

\bibitem[Mei et~al., 2021b]{mei2021leveraging}
Mei, J., Gao, Y., Dai, B., Szepesvari, C., and Schuurmans, D. (2021b).
\newblock Leveraging non-uniformity in first-order non-convex optimization.
\newblock In {\em International Conference on Machine Learning}, pages 7555--7564. PMLR.

\bibitem[Mei et~al., 2020]{mei2020global}
Mei, J., Xiao, C., Szepesvari, C., and Schuurmans, D. (2020).
\newblock On the global convergence rates of softmax policy gradient methods.
\newblock In {\em International Conference on Machine Learning}, pages 6820--6829. PMLR.

\bibitem[Munos, 2005]{munos2005error}
Munos, R. (2005).
\newblock Error bounds for approximate value iteration.
\newblock In {\em Proceedings of the National Conference on Artificial Intelligence}, volume~20, page 1006. Menlo Park, CA; Cambridge, MA; London; AAAI Press; MIT Press; 1999.

\bibitem[Nikolakakis et~al., 2022]{nikolakakis2022beyond}
Nikolakakis, K.~E., Haddadpour, F., Karbasi, A., and Kalogerias, D.~S. (2022).
\newblock Beyond lipschitz: Sharp generalization and excess risk bounds for full-batch gd.
\newblock {\em arXiv preprint arXiv:2204.12446}.

\bibitem[Puterman, 2014]{puterman2014markov}
Puterman, M.~L. (2014).
\newblock {\em Markov decision processes: discrete stochastic dynamic programming}.
\newblock John Wiley \& Sons.

\bibitem[Raffin et~al., 2021]{stable-baselines3}
Raffin, A., Hill, A., Gleave, A., Kanervisto, A., Ernestus, M., and Dormann, N. (2021).
\newblock Stable-baselines3: Reliable reinforcement learning implementations.
\newblock {\em Journal of Machine Learning Research}, 22(268):1--8.

\bibitem[Schulman, 2015]{schulman2015trust}
Schulman, J. (2015).
\newblock Trust region policy optimization.
\newblock {\em arXiv preprint arXiv:1502.05477}.

\bibitem[Schulman et~al., 2015]{schulman2015high}
Schulman, J., Moritz, P., Levine, S., Jordan, M., and Abbeel, P. (2015).
\newblock High-dimensional continuous control using generalized advantage estimation.
\newblock {\em arXiv preprint arXiv:1506.02438}.

\bibitem[Schulman et~al., 2017]{schulman2017proximal}
Schulman, J., Wolski, F., Dhariwal, P., Radford, A., and Klimov, O. (2017).
\newblock Proximal policy optimization algorithms.
\newblock {\em arXiv preprint arXiv:1707.06347}.

\bibitem[Shani et~al., 2020]{shani2020adaptive}
Shani, L., Efroni, Y., and Mannor, S. (2020).
\newblock Adaptive trust region policy optimization: Global convergence and faster rates for regularized mdps.
\newblock In {\em Proceedings of the AAAI Conference on Artificial Intelligence}, volume~34, pages 5668--5675.

\bibitem[Sutton, 2018]{sutton2018reinforcement}
Sutton, R.~S. (2018).
\newblock Reinforcement learning: An introduction.
\newblock {\em A Bradford Book}.

\bibitem[Sutton et~al., 1999]{sutton1999policy}
Sutton, R.~S., McAllester, D., Singh, S., and Mansour, Y. (1999).
\newblock Policy gradient methods for reinforcement learning with function approximation.
\newblock {\em Advances in neural information processing systems}, 12.

\bibitem[Todorov et~al., 2012]{todorov2012mujoco}
Todorov, E., Erez, T., and Tassa, Y. (2012).
\newblock Mujoco: A physics engine for model-based control.
\newblock In {\em 2012 IEEE/RSJ International Conference on Intelligent Robots and Systems}, pages 5026--5033. IEEE.

\bibitem[Tomar et~al., 2020]{tomar2020mirror}
Tomar, M., Shani, L., Efroni, Y., and Ghavamzadeh, M. (2020).
\newblock Mirror descent policy optimization.
\newblock {\em arXiv preprint arXiv:2005.09814}.

\bibitem[Vaswani et~al., 2021]{vaswani2021general}
Vaswani, S., Bachem, O., Totaro, S., M{\"u}ller, R., Garg, S., Geist, M., Machado, M.~C., Castro, P.~S., and Roux, N.~L. (2021).
\newblock A general class of surrogate functions for stable and efficient reinforcement learning.
\newblock {\em arXiv preprint arXiv:2108.05828}.

\bibitem[Vaswani et~al., 2024]{vaswani2024decision}
Vaswani, S., Kazemi, A., Babanezhad~Harikandeh, R., and Le~Roux, N. (2024).
\newblock Decision-aware actor-critic with function approximation and theoretical guarantees.
\newblock {\em Advances in Neural Information Processing Systems}, 36.

\bibitem[Williams, 1992]{williams1992simple}
Williams, R.~J. (1992).
\newblock Simple statistical gradient-following algorithms for connectionist reinforcement learning.
\newblock {\em Machine learning}, 8(3-4):229--256.

\bibitem[Xiao, 2022]{xiao2022convergence}
Xiao, L. (2022).
\newblock On the convergence rates of policy gradient methods.
\newblock {\em Journal of Machine Learning Research}, 23(282):1--36.

\bibitem[Yuan et~al., 2023]{yuan2023linear}
Yuan, R., Du, S.~S., Gower, R.~M., Lazaric, A., and Xiao, L. (2023).
\newblock Linear convergence of natural policy gradient methods with log-linear policies.
\newblock In {\em International Conference on Learning Representations}.

\bibitem[Zhong and Zhang, 2024]{zhong2024theoretical}
Zhong, H. and Zhang, T. (2024).
\newblock A theoretical analysis of optimistic proximal policy optimization in linear markov decision processes.
\newblock {\em Advances in Neural Information Processing Systems}, 36.

\end{thebibliography}
